%% file: psi-manifold.tex
\documentclass[accepted]{uai2022}

\usepackage[american]{babel}
\usepackage{times}
\usepackage{soul}
\usepackage{url}
\usepackage{hyperref}
\usepackage[utf8]{inputenc}
\usepackage{graphicx}
\usepackage{amsmath}
\usepackage{amsthm}
\usepackage{booktabs}
\usepackage{algorithm}
\usepackage{algorithmic}
\usepackage{multirow}
\usepackage{tikz}

\usepackage{microtype}      
\usepackage{subcaption}
\usepackage{amssymb}
\usepackage{mathtools}
\usepackage{natbib}

\usepackage{thmtools}
\usepackage{thm-restate}

\newcommand{\bw}{\text{\boldmath{$w$}}}
\newcommand{\bW}{\text{\boldmath{$W$}}}

\newcommand{\btheta}{\text{\boldmath{$\theta$}}}

\newcommand{\bz}{\boldsymbol{z}}
\newcommand{\bx}{\boldsymbol{x}}

\newcommand{\bxi}{\boldsymbol{\xi}}
\newcommand{\bXi}{\boldsymbol{\Xi}}
\newcommand{\bzeta}{\boldsymbol{\zeta}}

\newcommand{\ba}{\boldsymbol{a}}
\newcommand{\bu}{\boldsymbol{u}}

\newcommand{\bg}{\boldsymbol{g}}
\newcommand{\by}{\boldsymbol{y}}

\newcommand{\bU}{\boldsymbol{U}}

\newcommand{\bV}{\boldsymbol{V}}

\newcommand{\bbR}{\mathbb{R}}

\newcommand{\grad}{\mathrm{grad}}

\newcommand{\bv}{\boldsymbol{v}}

\newtheorem{definition}{\textbf{Definition}}\newtheorem{lemma}{\textbf{Lemma}}\newtheorem{theorem}{\textbf{Theorem}}\newtheorem{proposition}{\textbf{Proposition}}\newtheorem{remark}{\textbf{Remark}}

\newcommand{\mE}{\mathbb{E}}

\newcommand{\cM}{\mathcal{M}}

\newcommand{\cL}{\mathcal{L}}

\newcommand{\cN}{\mathcal{N}}

\newcommand{\cT}{\mathcal{T}}

\newcommand{\cG}{\mathcal{G}}

\makeatother

\usepackage{authblk}

\title{Accelerating Training of Batch Normalization: A Manifold Perspective}
\author[1,2]{Mingyang Yi \thanks{yimingyang17@mails.ucas.edu.cn}}
\affil[1]{University of Chinese Academy of Sciences}
\affil[2]{Academy of Mathematics and Systems Science, Chinese Academy of Sciences}
\date{}
\begin{document}
	
	\maketitle
	\author{}
	\begin{abstract}
		Batch normalization (BN) has become a critical component across diverse deep neural networks. The network with BN is invariant to positively linear re-scale transformation, which makes there exist infinite functionally equivalent networks with different scales of weights. However, optimizing these equivalent networks with the first-order method such as stochastic gradient descent will obtain a series of iterates converging to different local optima owing to their different gradients across training. To obviate this, we propose a quotient manifold \emph{PSI manifold}, in which all the equivalent weights of the network with BN are regarded as the same element. Next, we construct gradient descent and stochastic gradient descent on the proposed PSI manifold to train the network with BN. The two algorithms guarantee that every group of equivalent weights (caused by positively re-scaling) converge to the equivalent optima. Besides that, we give convergence rates of the proposed algorithms on the PSI manifold. The results show that our methods accelerate training compared with the algorithms on the Euclidean weight space. Finally, empirical results verify that our algorithms consistently improve the existing methods in both convergence rate and generalization ability under various experimental settings.
	\end{abstract}      
	\section{Introduction}
	Batch normalization (BN) \citep{ioffe2015batch} is one of the most critical innovations in deep learning, which appears to help optimization as well as generalization \citep{ioffe2015batch, santurkar2018does}. Despite the success of BN, there is one adverse effect in terms of optimization due to the positively scale-invariant (PSI) property brought by it. The PSI property is explained as the weights in each layer with BN are invariant to positively linear re-scaling. Due to this, there can be an infinite number of networks functionally equivalent to each other but with various scales of weights. These networks can converge to different local optima owing to different gradients \citep{cho2017riemannian,huang2017projection}. Then the converged point can be sensitive to the scale of parameters, and some of them may have poor performances (see Section \ref{sec:PSI manifold}). Hence, it is desirable to obviate such ambiguity of training. 
	\par
	To this end, we leverage the technique of optimization on manifold \citep{absil2009optimization,cho2017riemannian,huang2017projection,badrinarayanan2015symmetry}. We propose to constrain the scale-invariant weights of deep neural networks with BN in a quotient manifold i.e. \emph{PSI manifold} defined in Section \ref{sec:PSI manifold}, in which all the positively scale-equivalent weights are viewed as the same element. Constraining the scale-invariant weights on the PSI manifold maintains the representation ability of hypothesis space due to the PSI property of network with BN. More importantly, optimizing on the manifold breaks the training ambiguity caused by the PSI property.
	\par
	By constructing the Riemannian metric and retraction function \citep{absil2009optimization} on such PSI manifold \citep{boumal2019global}, we propose the gradient descent (GD), stochastic gradient descent (SGD), and SGD with momentum on such manifold. We abbreviate the corresponding algorithms as PSI-GD, PSI-SGD, and PSI-SGDM, respectively. Compared with vanilla GD and SGD on the Euclidean weight space, the proposed algorithms guarantee the equivalent scale-invariant weights (caused by positively re-scaling) converge to the equivalent local optima. 
	\par
	
	\begin{figure*}[t!]\centering
		\includegraphics[width=0.8\textwidth]{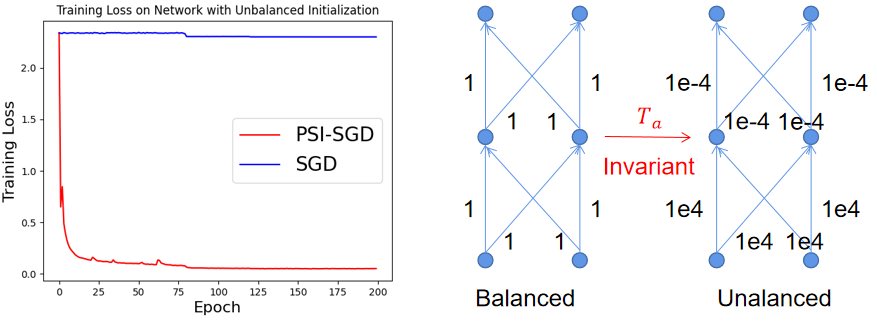}
		\caption{The figure on the right hand side are two batch normalized two-layer neural networks with weights $\bW$ and $T_{\ba}(\bW)$, where $\bW$ and $T_{\ba}(\bW)$ are respectively balanced and unbalanced weights. The PSI property says that $f(\bx, \bW)=f(\bx,T_{\ba}(\bW))$ with $\bW = (1,\cdots, 1)^{T}$ and $\ba=(\rm{1e4, 1e4, 1e-4, 1e-4})>0$. The figure on the left hand side is the training loss of our PSI-SGD and SGD under an unbalanced initialization.}
		\label{fig:network}
	\end{figure*}

	In contrast to the literature of optimizing network with BN while constraining parameters on manifold \citep{cho2017riemannian,huang2017projection,badrinarayanan2015symmetry,badrinarayanan2015symmetry}, we give the convergence rates of our PSI-GD and PSI-SGD. The convergence rates of the two methods are respectively in the order of $O(1/T)$ and $O(1/\sqrt{T})$ under the non-convex and smooth assumptions, where $T$ is the number of iterations. The results match the optimal one with fine-tuned learning rates \citep{ghadimi2013stochastic} under non-convex optimization problem. Besides that, the proposed algorithms are actually training with adaptive learning rate, which is decided by the local smoothness of loss function. The adaptive learning rate accelerates the training process as our proposed algorithms have improved constant dependence in the convergence rates, compared with their vanilla versions on the Euclidean space. 
	\par
	To empirically study the proposed algorithms, we compare PSI-SGD and PSI-SGDM with the other three baseline algorithms. They are respectively vanilla SGD, adaptive learning rate algorithm Adam \citep{kingma2014adam}, and another manifold based algorithm SGDG \citep{cho2017riemannian}. The experiments are conducted on image classification task on datasets \texttt{CIFAR} and \texttt{ImageNet} \citep{krizhevsky2009learning,deng2009imagenet}. We empirical observe the proposed methods consistently improve the baseline algorithms in the sense of both convergence rate and generalization over various experimental settings. The observed accelerated convergence rate and better generalization ability of PSI-GD and PSI-SGD verify our theoretical justification that the two method are able to find local minima with less iterations. 
	\section{Related Works}
	\paragraph{Optimization on Manifold.} 
	\begin{figure}[h]\centering
		\includegraphics[width=0.5\textwidth]{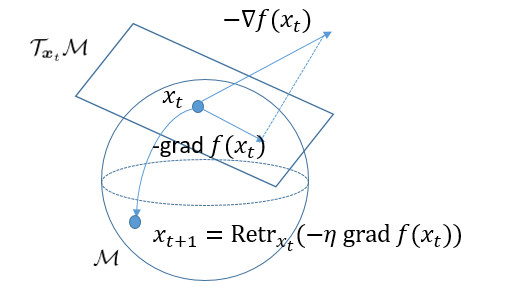}
		\caption{Gradient descent on Riemannian manifold.}
		\label{fig:gd on manifold}
	\end{figure}
	\cite{absil2009optimization} presents an introduction and summarization to this topic. The foundations and some recent theoretical results in this topic refers to   \citep{absil2006joint,liu2017accelerated,zhang2016first,boumal2019global}. Roughly speaking, optimization on manifold converts the constrained problem into an unconstrained problem on a specific manifold. \cite{huang2018orthogonal,lezcano2019cheap,casado2019trivializations,li2019efficient} leverage the technique to handle some explicit constraints to deep neural networks e.g., orthogonal constraints. \citep{cho2017riemannian,huang2017projection,badrinarayanan2015symmetry} fix the ambiguity of optimizing network with BN by manifold based algorithms. However, in contrast to this paper, some of the theoretical properties of their methods are missing, e.g., the convergence rate. 
	
	\paragraph{PSI property.} Our method in this paper is build upon PSI property brought by BN. The PSI property properly defined in Section \ref{sec:PSI manifold} are widely appears in modern deep neural networks. For example, the network with ReLU activation \citep{neyshabur2015path} and the network with normalization layer, e.g., batch normalization \citep{ioffe2015batch}, layer normalization \citep{ba2016layer} and group normalization \citep{wu2018group}.  \cite{arora2018theoretical,wu2018wngrad} use a similar theoretical framework as ours to show the network with PSI property allows a more flexible choice of the learning rate, and thus easier to be trained. However, we design algorithm to obviate the optimization brought by PSI property, and show the proposed methods can have improved convergence rate. 
	\par
	Due to the PSI property, \cite{li2019exponential} uses exponentially increased learning rate for the network with BN to handle implicit effect brought by weight decay. In contrast to fixing the ambiguity of optimization by manifold based methods, \cite{neyshabur2015path,meng2018g,zheng2019capacity} re-parameterize the network by ``path value'' and optimizing on the path value space. 
	\section{Background}
	\subsection{Problem Set Up}\label{sec:batch}
	In this paper, we consider the deep neural network with the following structure
	\begin{equation}\small\label{eq1}
	f(\bx, \bW) = BN(\bW_{L}\phi(BN(\bW_{L-1}\phi(\cdots \phi(BN(\bW_{1}\bx)))))),
	\end{equation}
	where $\phi(\cdot)$ is the non-linear activation, and $\bW_{l}$ is the weight matrix of the $l$-th layer. $BN(\cdot)$ operator in the equation represents batch normalization \citep{ioffe2015batch} layer, which normalizes every hidden output $z=\btheta^{T}\bx$ across a batch of training samples as 
	\begin{equation}\small
	BN(\bz)=\gamma\frac{\bz-\mathbb{E}(\bz)}{\sqrt{\mathrm{Var}(\bz)}} + \beta= \gamma\frac{\btheta^{T}(\bx-\mathbb{E}(\bx))}{\sqrt{\btheta^{T}\mathrm{Var}_{\bx}\btheta}} + \beta,
	\end{equation} 
	where $\gamma$ and $\beta$ are learnable affine parameters. The loss is  
	\begin{equation}
	\small
	\cL(\bW,\bg) = \frac{1}{n}\sum\limits_{i=1}^{n}\ell(f(\bx_{i},\bW),y_{i}),
	\end{equation}
	across this paper, where $\{(\bx_{i}, y_{i})\}$ is the training set, $\ell(\bx, y)$ represents the loss function, and $\bg$ are scale-variant parameters i.e. $\gamma$ and $\beta$ in BN layer. Let $\bW=(\bw^{(1)}\cdots,\bw^{(m)})$, the $\bw^{(i)}$ is the weight vector of the $i$-th neuron. Then $\cL(\bW,\bg)$ is positively scale-invariant w.r.t. $\bW$. It means $\cL(\bW,\bg)=\cL(T_{\ba}(\bW),\bg)$ for any $T_{\ba}(\bW)=(a_{1}\bw^{(1)}, \cdots, a_{m}\bw^{(m)})$ with $\{a_{i}>0\}$. In the right hand side of Figure \ref{fig:network}, we illustrate such PSI property (the network is invariant for rescaling operator $T_{\ba}(\cdot)$) via a two layer network, where the number between two circles (neurons) is the weight between them. Specially, let $\bV$ be the normalized $\bW$ which means $\bV=(\bw^{(1)}/\|\bw^{(1)}\| ,\cdots, \bw^{(m)}/\|\bw^{(m)}\|)$. Then we have $\cL(\bV, \bg)=\cL(\bW, \bg)$. 
	\par
	It worth noting that our discussions below can be applied on the top of any functions with PSI property e.g., ReLU neural networks \citep{meng2018g,neyshabur2015path}. However, we use NN with BN as a representation to illustrate our conclusions. 
	\subsection{Optimization on Manifold}
	We present a brief introduction to some definitions appeared in the optimization on manifold, more details are deferred to Appendix \ref{app: definitions of manifold} or \citep{absil2009optimization}. A Matrix Manifold $\mathcal{M}$ is a subset of Euclidean space that for any $\bx\in\mathcal{M}$, there exists a neighborhood $U_{\bx}$ of $\bx$, such that $U_{\bx}$ is homeomorphic to a Euclidean space. It has been pointed out that the PSI parameters $\bW$ can be defined on a matrix manifold, e.g., Grassmann manifold or Oblique manifold without losing representation capacity of model \citep{cho2017riemannian,huang2017projection}. For any $\bx\in\cM$, there is a tangent space $\cT_{\bx}\cM$, such that for any $f(\cdot)$ defined on $\cM$, its Riemannian gradient $\grad f(\bx)\in \cT_{\bx}\cM$ \footnote{According to \citep{absil2009optimization}, the Riemannian gradient $\grad f(\bx)$ sometimes is the projection of gradient $\nabla f(x)$ to the tangent space $\cT_{\bx}$.}. 
	\par
	With a given retraction function $R_{\bx}(\cdot): \cT_{\bx}\cM\rightarrow\cM$, it was shown in \citep{boumal2019global} that the gradient descent on manifold (refers to Figure \ref{fig:gd on manifold})
	\begin{equation}\label{eq:riemannian gradient descent}
	\small
	\bx_{t+1} = R_{\bx_{t}}\left(-\eta\grad f(\bx_{t})\right)
	\end{equation}
	has the same convergence rate with vanilla gradient descent in Euclidean space under some mild assumptions. Roughly speaking, the Riemannian gradient descent creates a series of iterates moving along the direction of $\grad f(\bx)$, while the moving is conducted via the pre-defined retraction function $R_{\bx}(\cdot)$. 
	\section{PSI Manifold}\label{sec:PSI manifold}
	The optimization ambiguity of PSI parameters states as follows. By chain rule, 
	\begin{equation}\label{eq:gradient scale}
	\small
	\nabla_{\bw^{(i)}}\cL(\bW,\bg) = a_{i}\nabla_{a_{i}\bw^{(i)}}\cL(T_{\ba}(\bW), \bg),
	\end{equation}
	so the two positively scale-equivalent points $\bW$ and $T_{\ba}(\bW)$ with the same output converge to different local optima owing to the different gradients. In other words, gradient-based algorithms in Euclidean space can not guarantee two iterates started from $\bW$ and $T_{\ba}(\bW)$ positively scale equivalent with each other across the training. This makes the converged point becomes sensitive the scale of weights. 
	\par
	Unfortunately, the sensitivity sometimes makes model converge to a local minima with poor performance. For example, we train a simple two-layer neural network with BN on \texttt{MNIST} \citep{lecun1998gradient}, which is a image classification task. The network is trained with a unbalanced initialization where the unbalanced here means the weights in each layer are in different scales. As can be seen, in the left hand side of Figure \ref{fig:network}, in contrast to our PSI-SGD (see Section \ref{sec:optimization on PSI manifold}), SGD performs very poorly under such initialization. However, the poor performance of SGD does adopted under balanced initialization \citep{lecun1998gradient}. Thus the optimization ambiguity can break the success of training.  
	\par
	To obviate this, we construct a quotient manifold, i.e., PSI manifold. Constraining the model's parameters in such manifold can fix the training ambiguity brought by the positively scale-invariant property.
	\subsection{Construction of PSI Manifold}   
	A quotient manifold is induced by an equivalent relationship. We formally formulate the positively scale-equivalence to derive the PSI manifold with PSI parameters defined in it.  
	\begin{definition}[Positively Scale-Equivalent]
		For a pair of $\bW, \bW^{\prime}$, they are positively scale-equivalent with each other, if there exists a transformation $T_{\ba}(\cdot)$ such that $\bW^{\prime}=T_{\ba}(\bW)$. $\bW\sim \bW^{\prime}$ denotes the equivalence.
	\end{definition}
	We can verify that the positively scale-equivalent ``$\sim$'' is an equivalent relationship. The Proposition 3.4.6 in \citep{absil2009optimization} implies the following proposition which constructs the PSI manifold. 
	\begin{proposition}
		For PSI parameters $\bW$, the positively scale-equivalent relationship $\sim$ induces a quotient manifold $\cM = \overline{\mathcal{M}}/\sim = \bbR^{d_{1}}\times\cdots\bbR^{d_{m}}/\sim$ named PSI manifold. Here $d_{i}$ is the dimension of $\bw^{(i)}$. 
	\end{proposition}
	\par
	We formally defined the PSI manifold. Intuitively, all the parameters that are positively scale-equivalent with each other in the Euclidean space are viewed as the same element in the PSI manifold $\mathcal{M}$. Then, we can constrain the PSI parameters directly in the manifold without losing the representation ability of the model. 
	\par 
	The PSI manifold is proposed to obviate the ambiguity of optimization brought by the PSI property. To this end, we can optimize the PSI parameters directly in the PSI manifold, since all the positively scale-equivalent parameters are the same point in it. But \eqref{eq:riemannian gradient descent} implies that optimization on PSI manifold requires the formulation of Riemannian gradient and retraction function on it. The following proposition defines a Riemannian metric in the PSI manifold, which formulates the Riemannian gradient (See the definition in  Appendix \ref{app: definitions of manifold}). 
	\begin{proposition}\label{pro:riemannian metric}
		For every $\bW\in\mathcal{M}$ \footnote{Please note that any $T_{\ba}(\bW)$ and $\bW$ are view as the same element in the PSI manifold $\mathcal{M}$. Here $\bW$ is a representation of set $\{T_{\ba}(\bW): a_{i} > 0 \}$.}, the Riemannian metric $\langle\cdot,\cdot\rangle_{\bW}$ in the PSI manifold can be defined as
		\begin{equation}\small\label{eq:riemannian metric}
			\langle \bXi_{1}, \bXi_{2} \rangle_{\bW} = \sum\limits_{i=1}^{m}\frac{\langle\bxi_{1}^{(i)}, \bxi_{2}^{(i)}\rangle}{\|\bw^{(i)}\|^{2}},
		\end{equation}
		where $\bXi_{1}, \bXi_{2}\in \cT_{\bW}\mathcal{M}$, with $\bXi_{k} = (\bxi_{k}^{(1)},\cdots\bxi^{(m)}_{k})$. 
	\end{proposition}
	The proof of this Proposition is in Appendix \ref{app:proof of riemannian metric}. With the Riemannian metric, we can exactly compute the Riemannian gradient on the PSI manifold. According to Definition of Riemannian gradient in Appendix \ref{app: definitions of manifold}, we have 
	\begin{equation}
	\small
	\grad_{\bw^{(i)}} \cL(\bW,\bg) = \|\bw^{(i)}\|^{2}\nabla_{\bw^{(i)}}\cL(\bW, \bg)
	\end{equation}
	for each $\bw^{(i)}$. The retraction function $R_{\bW}(\cdot)$ is required to obtain gradient-based algorithms on the PSI manifold. 
	\begin{proposition}\label{pro:retr}
		For every $\bW\in\mathcal{M}, \bXi\in \cT_{\bW}\mathcal{M}$, a retraction function $R_{\bW}(\bXi)$ on the PSI manifold can be 
		\begin{equation}\small
			R_{\bW}(\bXi) = \bW + \bXi.
		\end{equation}
	\end{proposition}
	The proof of this proposition appears in Appendix \ref{app:proof of unified}. Now we are ready to optimize PSI parameters in the PSI manifold directly, which obviates the ambiguity of training 
	\subsection{Optimization on the PSI manifold}\label{sec:optimization on PSI manifold}
	In this subsection, we give the update rules of GD, SGD, and SGD with momentum on the PSI manifold, abbreviated as PSI-GD, PSI-SGD, and PSI-SGDM respectively. We show that all of these update rules on the PSI manifold generate a unified optimization path for parameters that positively scale-equivalent with each other.
	\par
	Combining \eqref{eq:riemannian gradient descent} and Proposition \ref{pro:riemannian metric}, \ref{pro:retr}, we have the following update rule of GD
	\begin{equation}\label{eq:GD on manifold}
	\small
	\begin{aligned}
	\bw_{t + 1}^{(i)} & = R_{\bw^{(i)}_{t}}\left(-\eta_{\bw_{t}^{(i)}}\grad_{\bw_{t}^{(i)}}\cL(\bW_{t}, \bg_{t})\right) \\
	& = \bw_{t}^{(i)} - \eta_{\bw_{t}^{(i)}}\|\bw_{t}^{(i)}\|^{2}\nabla_{\bw_{t}^{(i)}}\cL(\bW_{t},\bg_{t}), 
	\end{aligned}
	\end{equation}
	SGD
	\begin{equation}\label{eq:SGD on manifold}
	\small
	\begin{aligned}
	\bw_{t + 1}^{(i)}  & = R_{\bw^{(i)}_{t}}\left(-\eta_{\bw_{t}^{(i)}}\grad_{\bw_{t}^{(i)}}\hat{\cL}(\bW_{t}, \bg_{t})\right) \\
	& = \bw_{t}^{(i)} - \eta_{\bw_{t}^{(i)}}\|\bw_{t}^{(i)}\|^{2}\cG_{\bw_{t}^{(i)}}(\bW_{t},\bg_{t}), 
	\end{aligned}
	\end{equation}
	and SGD with momentum 
	\begin{equation}\small\label{eq:riemannian gradient descent with momentum}
	\begin{aligned}
		\bu_{t+1}^{(i)} & = R_{\rho\bu_{t}^{(i)}}\left(-\eta_{\bw_{t}^{(i)}}\grad_{\bw_{t}^{(i)}}\hat{\cL}(\bW_{t}, \bg_{t})\right) \\
		& = \rho\bu_{t}^{(i)} - \eta_{\bw_{t}^{(i)}}\|\bw_{t}^{(i)}\|^{2}\cG_{\bw_{t}^{(i)}}(\bW_{t},\bg_{t}); \\
		\bw_{t+1}^{(i)} & = R_{\bw_{t}^{(i)}}\left(\bu_{t+1}^{(i)}\right) = \bw_{t+1}^{(i)} +\bu_{t+1}^{(i)}.
		\end{aligned}
	\end{equation}
	for PSI parameters. Here 
	\begin{equation}\label{eq:PSI-SGD gradient estimator}
	\small
	\begin{aligned}
	\grad_{\bw_{t}^{(i)}} & \hat{\cL}(\bW_{t}, \bg_{t})  = \|\bw_{t}^{(i)}\|^{2}\cG_{\bw_{t}^{(i)}}(\bW_{t},\bg_{t}) \\
	& = \|\bw_{t}^{(i)}\|^{2}\nabla_{\bw_{t}^{(i)}}\frac{1}{S}\sum_{k=1}^{S}\nabla_{\bw_{t}^{(i)}}\ell\left(f(\bx_{k},\bW_{t}), y_{k}\right)
	\end{aligned}
	\end{equation}
	for a batch of $\{(\bx_{1},y_{1}),\cdots,(\bx_{S},y_{S})\}$ is an unbiased estimation to the Riemannian gradient. In addition, the update rules of non-scale-invariant parameters $\bg$ follows the gradient-based algorithms in Euclidean space. 
	\begin{remark}
		The proposed PSI-GD is similar to the one in \citep{badrinarayanan2015symmetry}. However, they do not prove that the update rule is induced by a retraction function. Besides that, our theoretical characterization of its convergence rate is absent in their work. 
	\end{remark}
	\par
	One can justify that the proposed algorithms on PSI manifold are essentially using adaptive learning rates decided by $\|\bw_{t}^{(i)}\|^{2}$ to match the gradient scale brought by a variant of weight scale. The following proposition shows the well-posedness of the algorithms on PSI manifold. 
	\begin{theorem}\label{thm:thm2}
		For two positively scale-equivalent weights $\bW_{0}$ and $T_{\ba}(\bW_{0})$, let $\bW_{t}$, $\hat{\bW}_{t}$ be $t$-th iterate of PSI-SGDM respectively started from $\{\bW_{0},\bU_{0}\}$ and $\{T_{\ba}(\bW_{0}), T_{\ba}(\bU_{0})\}$. Then we have $\hat{\bW}_{t}= T_{\ba}(\bW_{t})$.
	\end{theorem}
	This theorem can be easily generalized to PSI-GD and PSI-SGD; the proof of it is in Appendix \ref{app:proof of unified}. The conclusion shows that the iterates generated by the update rules on the PSI manifold are equivalent with respect to applying $T_{\ba}(\cdot)$. Hence, optimizing on the PSI manifold obviates the optimization ambiguity brought by the PSI property. The complete algorithm of PSI-SGDM refers to Algorithm \ref{alg1}. With $\rho=0$, the PSI-SGDM degenerates to PSI-SGD.  
	\begin{algorithm}[t!]
		\caption{SGD with momentum on the PSI manifold (PSI-SGDM).}
		\label{alg1}
		\begin{algorithmic}\small
			\STATE {\textbf{Input}: Training steps $T$; batch size $S$; momentum parameters $\rho$; learning rate $\eta_{\bw_{t}^{(i)}}$ and $\eta_{\bg_{t}}$.}
			\FOR {$t = 0 \cdots T - 1$}
			\STATE {Sampling a batch of data $\{(x_{1},y_{1}),\cdots, (\bx_{S}.y_{S})\}$ from training set}
			\FOR   {$i=1,\cdots, m$}
			\STATE {$\bu_{t+1}^{(i)} = \newline \rho\bu_{t}^{(i)} - \eta_{\bw_{t}^{(i)}}\|\bw_{t}^{(i)}\|^{2}\frac{1}{S}\sum_{k=1}^{S}\nabla_{\bw_{t}^{(i)}}\ell\left(f(\bx_{k},\bW_{t}), y_{k}\right)$}
			\STATE {$\bw_{t+1}^{(i)} = \bw_{t+1}^{(i)} + \bu_{t+1}^{(i)}$} 
			\ENDFOR
			\STATE {$\bg_{t + 1} = \bg_{t} - \eta_{\bg_{t}}\sum_{k=1}^{S}\nabla_{\bg_{t}}\ell\left(f(\bx_{k},\bW_{t}), y_{k}\right)$}
			\ENDFOR
			\RETURN {$(\bW_{T}, \bg_{T})$}
		\end{algorithmic}
	\end{algorithm}
	\par
	\begin{remark}
		The update rule of SGD with momentum on manifold \citep{cho2017riemannian,liu2017accelerated} involves the parallel transformation to make sure the  $\bU_{t+1}$ locates in the tangent space $\cT_{\bW_{t}}\mathcal{M}$. However, it requires the proposed retraction in Proposition \ref{pro:retr} to be an exponential retraction map \citep{absil2009optimization}, which may fail in-practical. Hence, we heuristically propose the PSI-SGDM without verifying the retraction function. Even though, our method is well-defined to positively scale-transformation.
	\end{remark}
	\section{Optimization on PSI Manifold Accelerates Training}\label{sec:convergence results}
	In this section, we give the convergence rates of PSI-GD and PSI-SGD, and show that they accelerate training compared with vanilla GD and SGD on the Euclidean space. 
	\subsection{Convergence Results}
	We give the convergence rates of PSI-GD and PSI-SGD (Algorithm \ref{alg1}) in this subsection. For the PSI parameters,  
	let $\bV_{t}=(\bw^{(1)}_{t}/\|\bw^{(1)}_{t}\| ,\cdots, \bw^{(m)}_{t}/\|\bw^{(m)}_{t}\|)$ be the normalized parameters, we impose the following assumptions to characterize the smoothness of loss function.  
	\begin{equation}\label{eq:lipschitz on manifold}\small
	\begin{aligned}
	 & \left\|\nabla^{2}_{\bv^{(i)}\bv^{(j)}}\cL(\bV,\bg)\right\|_{2}  \leq L_{ij}^{\bv\bv}, \\
	 & \left\|\nabla^{2}_{\bv^{(i)}\bg}\cL(\bV,\bg)\right\|_{2}  \leq L_{i}^{\bv\bg}, \\
	 & \left\|\nabla^{2}_{\bg}\cL(\bV,\bg)\right\|_{2}  \leq L^{\bg\bg}.
	\end{aligned}
	\end{equation}
	Here $\|\cdot\|_{2}$ is the spectral norm of matrix. For PSI-SGD, we further impose the bounded variance assumption:
	\begin{equation}\label{eq:bounded variance}
	\small
	\begin{aligned}
	\mE\left[\left\|\cG_{\bw_{t}^{(i)}}(\bW_{t},\bg_{t}) - \nabla_{\bw_{t}^{(i)}}\cL(\bW_{t},\bg_{t}) \right\|^{2}\right] & \leq \sigma^{2}; \\
	\mE\left[\left\|\cG_{\bg_{t}}(\bW_{t},\bg_{t}) - \nabla_{\bw_{t}^{(i)}}\cL(\bW_{t},\bg_{t}) \right\|^{2}\right] & \leq \sigma^{2}. 
	\end{aligned}
	\end{equation}
	\begin{remark}
		Due to \eqref{eq:gradient scale}, the upper bound \eqref{eq:bounded variance} can not hold as $\|\bw_{i}\|$ goes to zero. However, due to Lemma \ref{lem:increasing} below, the weights $\bw_{i}$ obtained by GD or SGD will never goes to zero. Thus the bounded variance \eqref{eq:bounded variance} is reasonable in this regime. 
	\end{remark}
	Let $\cL(\bW^{*},\bg^{*})=\inf_{\bW,\bg}\cL(\bW,\bg)$, the following theorems give the convergence rates of PSI-GD and PSI-SGD.
	\begin{theorem}\label{thm:PSI-GD}
		Let $\{\bW_{t},\bg_{t}\}$ be the iterates of PSI-GD \eqref{eq:GD on manifold}, then we have 
		\begin{equation}\label{eq:rates of PSI-GD}
		\small
		\min_{0\leq t< T}\left\|\nabla_{\bV_{t},\bg_{t}}\cL(\bV_{t}, \bg_{t})\right\|^{2} \leq \frac{2\tilde{L}(\cL(\bW_{0},\bg_{0}) - \cL(\bW^{*},\bg^{*}))}{T}
		\end{equation}
		by choosing $\eta_{\bw_{t}}^{(i)} = 1/\tilde{L}_{\bv^{(i)}}$ and $\eta_{\bg_{t}} = 1/\tilde{L}_{\bg}$. Here $\tilde{L} = \max\{\tilde{L}_{\bv^{(1)}}, \cdots, \tilde{L}_{\bv^{(m)}}, \tilde{L}_{\bg}\}$, where $\tilde{L}_{\bv^{(i)}} = L_{i}^{\bv\bg} + \sum_{j=1}^{m}L_{ij}^{\bv\bv}; \tilde{L}_{\bg} = mL^{\bg\bg} + \sum_{i=1}^{m}L_{i}^{\bv\bg}$. 
	\end{theorem}
	\begin{theorem}\label{thm:PSI-SGD}
		Let $\{\bW_{t},\bg_{t}\}$ be the iterates of PSI-SGD \eqref{eq:SGD on manifold}, then we have 
		\begin{equation}\label{eq:rates of PSI-SGD}
		\small
		\begin{aligned}
		\min_{0\leq t< T}\!\mE\!\left[\left\|\nabla_{\bV_{t},\bg_{t}} \cL(\bV_{t}, \bg_{t})\right\|^{2}\right] & \!\!\leq\!\! \frac{2\tilde{L}(\cL(\bW_{0}, \bg_{0}) \!-\! \cL(\bW^{*}, \bg^{*}))}{\sqrt{T}} \\
		& + \frac{\sigma^{2}\tilde{L}}{2\sqrt{T}}\sum\limits_{i=1}^{m}\left(\frac{1}{\tilde{L}_{\bg}} + \frac{1}{\tilde{L}_{\bv^{(i)}}}\right),
		\end{aligned}
		\end{equation}
		by choosing $\eta_{\bw_{t}^{(i)}}=\frac{1}{\tilde{L}_{\bv^{(i)}}\sqrt{T}}$ and $\eta_{\bg_{t}}=\frac{1}{\tilde{L}_{\bg}\sqrt{T}}$, where $\tilde{L}, \tilde{L}_{\bv^{(i)}}$ and $\tilde{L}_{\bg}$ are defined in Theorem \ref{thm:PSI-GD}.
	\end{theorem}
	The theorems are respectively proved in Appendix \ref{app:proof of PSI-GD} and \ref{app:proof of PSI-SGD}. The convergence rates of PSI-GD and PSI-SGD are respectively $O(1/T)$ and $O(1/\sqrt{T})$. The convergence rates of PSI-SGDM can be also obtained by combining the technique in \citep{yang2016unified} and our proof of Theorem \ref{thm:PSI-SGD}.    
	\par
	One may note the convergence rate is computed on the gradient w.r.t. normalized parameters $\bV_{t}$. But \eqref{eq:gradient scale} implies 
	\begin{equation}\small\label{eq:smoothness}
	\begin{aligned}
		 \left\|\nabla_{\bW_{t},\bg_{t}} \cL(\bW_{t}, \bg_{t})\right\|^{2} & = \sum\limits_{i=1}^{m}\frac{1}{\|\bw^{(i)}_{t}\|^{2}} \left\|\nabla_{\bv_{t}^{(i)},\bg_{t}} \cL(\bV_{t}, \bg_{t})\right\|^{2}\\
		& + \left\|\nabla_{\bg_{t}} \cL(\bV_{t}, \bg_{t})\right\|^{2} \\
		& \leq \left\|\nabla_{\bV_{t},\bg_{t}} \cL(\bV_{t}, \bg_{t})\right\|^{2},
	\end{aligned}
	\end{equation}
	if $\|\bw^{(i)}_{t}\|^{2} \geq 1$ for $1\leq i\leq m$. Thus we can get the corresponded convergence rates of PSI-GD and PSI-SGD. They match the optimal results with well tuned learning rates in the Euclidean space \citep{ghadimi2013stochastic}. The following lemma shows that the $\|\bw_{t}^{(i)}\|^{2}$ keeps increasing across training for gradient-based algorithms.
	\begin{lemma}[Lemma 2.4 in \citep{arora2018theoretical}]\label{lem:increasing}
		For any PSI weight $\bw^{(i)}$, $\bw^{(i)}$ and $\nabla_{\bw^{(i)}}\ell\left(f(\bx_{k},\bW), y_{k}\right)$ are perpendicular for any $(\bx_{k}, y_{k})$. On the other hand 
		\begin{equation}\small\label{eq:increasing}
		\begin{aligned}
		& \left\|\bw^{(i)} + \eta_{\bw^{(i)}}\nabla_{\bw^{(i)}}\ell\left(f(\bx_{k},\bW), y_{k}\right)\right\|^{2}  = \left\|\bw^{(i)}\right\|^{2} \\
		& + \eta_{\bw^{(i)}}^{2}\left\|\nabla_{\bw^{(i)}}\ell\left(f(\bx_{k},\bW), y_{k}\right)\right\|^{2}. 
		\end{aligned}
		\end{equation}
	\end{lemma}
	Thus, $\|\bw^{(i)}_{t}\|^{2}\geq 1$ holds, if $\|\bw^{(i)}_{0}\|^{2}\geq 1$. Since the network is usually initialized by $\bw^{(i)}\sim \cN(0,2/\sqrt{d_{i}})$ \citep{he2015delving}, $\|\bw^{(i)}_{0}\|^{2}\approx 2$ with high probability. Hence, we can conclude $\left\|\nabla_{\bW_{t},\bg_{t}} \cL(\bW_{t}, \bg_{t})\right\|^{2} \leq \left\|\nabla_{\bV_{t},\bg_{t}} \cL(\bV_{t}, \bg_{t})\right\|^{2}$. We will show the inequality and the increasing  $\|\bw_{t}^{(i)}\|^{2}$ explain the acceleration of the algorithms on PSI manifold.    
	\subsection{Why Optimization on the PSI manifold Accelerates Training}\label{sec: smoother function}
	\begin{figure*}[t!]\centering
		\begin{subfigure}[b]{0.24\textwidth}
			\includegraphics[width=\textwidth]{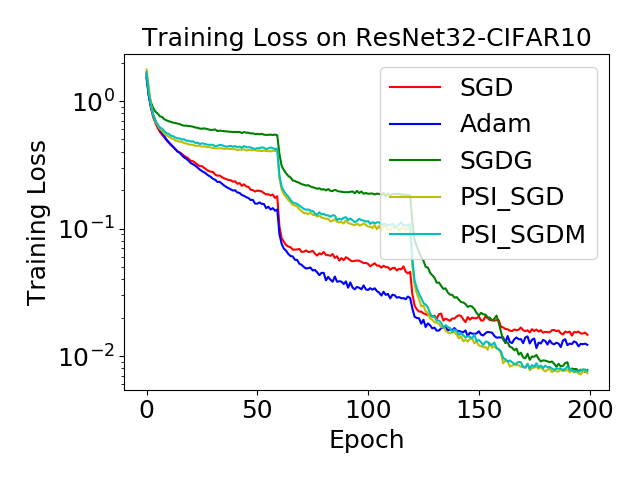}
		\end{subfigure}
		\begin{subfigure}[b]{0.24\textwidth}
			\includegraphics[width=\textwidth]{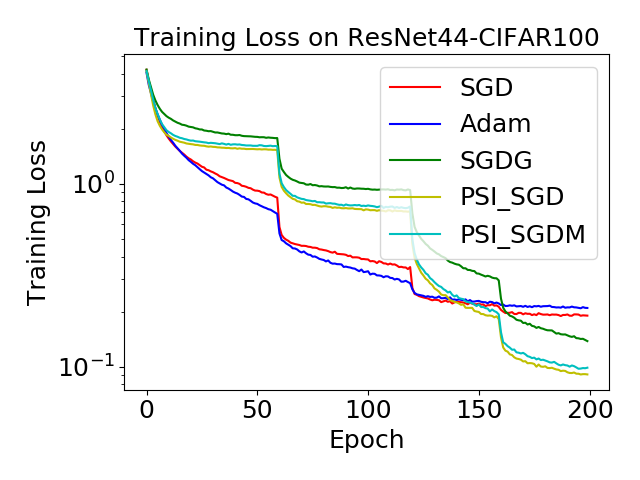}
		\end{subfigure}
		\begin{subfigure}[b]{0.24\textwidth}
			\includegraphics[width=\textwidth]{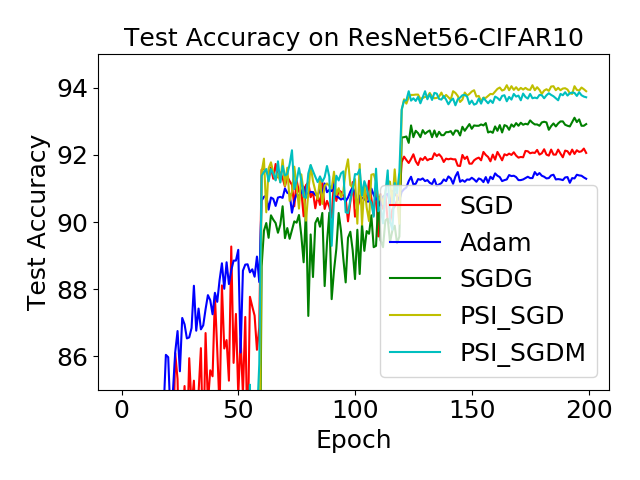}           
		\end{subfigure}    
		\begin{subfigure}[b]{0.24\textwidth}
			\includegraphics[width=\textwidth]{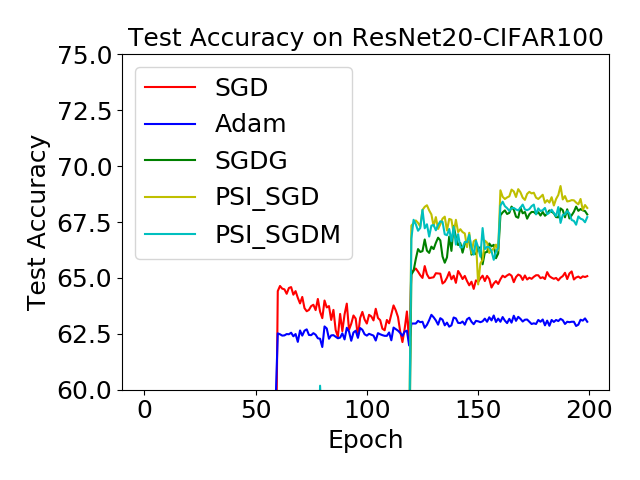}    
		\end{subfigure}
		\caption{Results of ResNet trained over various Algorithms on \texttt{CIFAR10} and \texttt{CIFAR100}.}
		\label{fig:result}
	\end{figure*}
	\begin{table*}[t!]
		\centering
		\caption{Performance of ResNet on various datasets. The ResNet for \texttt{CIFAR} and \texttt{ImageNet} are respectively the structure for the corresponded dataset, more details about the structure refers to \citep{he2016deep}. The results are the average over five (resp. three) independent runs for CIFAR (resp. ImageNet) with standard deviation reported.}
		\scalebox{1}{
			{
				\begin{tabular}{|c|ccccc|}
					\hline
					Dataset & \multicolumn{5}{|c|}{\texttt{CIFAR-10}} \\
					Algorithm & SGD           & Adam    & SGDG    & PSI-SGD        & PSI-SGDM    \\
					\hline
					ResNet20  & 91.14(±0.10)  &	89.90(±0.17)	&  91.38(±0.21)	  &   92.25(±0.03)	&  \textbf{92.41}(±0.12)     \\
					ResNet32  & 91.32(±0.23)  &	90.40(±0.32)	&  92.48(±0.35)	  &   \textbf{93.56}(±0.08)	&   93.30(±0.14)      \\
					ResNet44  & 91.95(±0.14)  &	91.02(±0.18)	&  92.99(±0.12)	  &   \textbf{93.90}(±0.16)	&   93.39(±0.11)      \\
					ResNet56  & 92.19(±0.21)  &	91.49(±0.12)	&  93.11(±0.16)	  &   \textbf{94.08}(±0.15)	&   93.90(±0.16)      \\
					\hline
					Dataset &  \multicolumn{5}{|c|}{\texttt{CIFAR-100}} \\
					\hline 
					ResNet20  & 65.53(±0.51)	&  63.35(±0.06)	&  68.19(±0.24)	  &   \textbf{69.11}(±0.07)	&  68.41(±0.21) \\
					ResNet32  & 67.41(±0.21)	&  64.43(±0.11)	&  70.13(±0.28)	  &   \textbf{70.81}(±0.08)	&  70.14(±0.13) \\ 
					ResNet44  & 67.72(±0.31)	&  65.03(±0.38)	&  71.32(±0.12)	  &    71.94(±0.21)	        &  \textbf{72.03}(±0.14)  \\
					ResNet56  & 68.02(±0.26)	&  65.77(±0.13)	&  71.46(±0.06)	  &   \textbf{72.88}(±0.24)	&  72.40(±0.16) \\ 
					\hline
					Dataset   &  \multicolumn{5}{|c|}{\texttt{ImageNet}} \\
					\hline
					ResNet18  & 67.72(±0.12)	&  68.16(±0.13)	&  68.97(±0.05)	  &   \textbf{70.38}(±0.06)	&  69.42(±0.06)   \\  
					ResNet34  & 71.30(±0.08)	&  70.68(±0.05)	&  72.40(±0.12)	  &   \textbf{73.31}(±0.08)	&  72.88(±0.12)   \\
					ResNet50  & 73.46(±0.14)	&  74.00(±0.16)	&  74.67(±0.14)	  &   74.67(±0.12)	        &  \textbf{75.11}(±0.05)  \\
					\hline
		\end{tabular}}}
		\label{tab:resnet_without_regularizer}
	\end{table*}
	Next, we show PSI-GD and PSI-SGD accelerate training compared with the algorithms on Euclidean space. 
	\par
	Generally, PSI parameters move towards a smoother region (smaller gradient Lipschitz constant), since the gradient Lipschitz constant is in inverse ratio to the weight scale which keeps increasing according to Lemma \ref{lem:increasing}. As smoother region allows a larger learning rate, gradually increasing the learning rate accelerates training. Fortunately, PSI-GD and PSI-SGD happen to be vanilla GD and SGD with adaptively increased learning rates according to \eqref{eq:GD on manifold}, \eqref{eq:SGD on manifold} and \eqref{eq:increasing}.
	\par
	To begin with, we verify the convergence rates of vanilla GD and SGD. We assume a global smoothness for $\cL(\bW,\bg)$,
	\begin{equation}\small\label{eq:lipschitz}
	\begin{aligned}
	& \left\|\nabla^{2}_{\bw^{(i)}\bw^{(j)}}\cL(\bW,\bg)\right\|_{2}  \leq L_{ij}^{\bw\bw}, \\
	& \left\|\nabla^{2}_{\bw^{(i)}\bg}\cL(\bW,\bg)\right\|_{2} \leq L_{i}^{\bw\bg},\\
	& \left\|\nabla^{2}_{\bg}\cL(\bW,\bg)\right\|_{2} \leq L^{\bg\bg}.
	\end{aligned}
	\end{equation}
	The assumption is stronger than \eqref{eq:lipschitz on manifold} since \eqref{eq:lipschitz on manifold} only involves the normalized parameters, thus  $L_{ij}^{\bv\bv}\leq L_{ij}^{\bw\bw}; L_{i}^{\bv\bg}\leq L_{i}^{\bw\bg}$. We consider the update rule of GD
	\begin{equation}\label{eq:GD}
	\small
	\begin{aligned}
	\bw_{t + 1}^{(i)} & = \bw_{t}^{(i)} - \eta_{\bw_{t}^{(i)}}\nabla_{\bw_{t}^{(i)}}\cL(\bW_{t},\bg_{t});\\
	\bg_{t + 1} & = \bg_{t} - \eta_{\bg_{t}}\nabla_{\bg_{t}}\cL(\bW_{t},\bg_{t}),
	\end{aligned}
	\end{equation}
	and SGD
	\begin{equation}\label{eq:SGD}
	\small
	\begin{aligned}
	\bw_{t + 1}^{(i)} &= \bw_{t}^{(i)} - \eta_{\bw_{t}^{(i)}}\cG_{\bw_{t}^{(i)}}(\bW_{t},\bg_{t});\\
	\bg_{t + 1} &= \bg_{t} - \eta_{\bg_{t}}\cG_{\bg_{t}}(\bW_{t},\bg_{t}),
	\end{aligned}
	\end{equation}
	where $\cG_{\bw_{t}^{(i)}}(\bW_{t},\bg_{t})$ and $\cG_{\bg_{t}}(\bW_{t},\bg_{t})$ are respectively unbiased estimations of $\nabla_{\bw_{t}^{(i)}}\cL(\bW_{t},\bg_{t})$ and $\nabla_{\bg_{t}}\cL(\bW_{t},\bg_{t})$ defined in \eqref{eq:PSI-SGD gradient estimator}. For SGD, we impose the bounded variance assumption \eqref{eq:bounded variance}. The following two Theorems characterize the convergence rate of GD and SGD.
	\begin{theorem}\label{pro:GD learning rate}
		Let $\{\bW_{t},\bg_{t}\}$ updated by GD \eqref{eq:GD}, by choosing $\eta_{\bw_{t}}^{(i)} = 1/\tilde{L}_{\bw^{(i)}}$ and $\eta_{\bg_{t}} = 1/\tilde{L}_{\bg}$,
		\begin{equation}\label{eq:rates of GD}
		\small
		\min_{0\leq t< T}\left\|\nabla_{\bW_{t},\bg_{t}} \cL(\bW_{t}, \bg_{t})\right\|^{2} \leq \frac{2\tilde{L}(\cL(\bW_{0},\bg_{0}) - \cL(\bW^{*},\bg^{*}))}{T}.
		\end{equation}
		Here $\tilde{L} = \max\{\tilde{L}_{\bw^{(1)}}, \cdots, \tilde{L}_{\bw^{(m)}}, \tilde{L}_{\bg}\}$, where $\tilde{L}_{\bw^{(i)}} = L_{i}^{\bw\bg} + \sum_{j=1}^{m}L_{ij}^{\bw\bw}; \tilde{L}_{\bg} = mL^{\bg\bg}+ \sum_{i=1}^{m}L_{i}^{\bw\bg}$.  
	\end{theorem}
	\begin{theorem}\label{pro:SGD learning rate}
		Let $\{\bW_{t},\bg_{t}\}$ updated by SGD \eqref{eq:SGD}, then
		\begin{equation}\label{eq:rates of SGD}
		\small
		\begin{aligned}
		\min_{0\leq t< T}\!\mE\!\left[\left\|\nabla_{\bW_{t},\bg_{t}} \cL(\bW_{t}, \bg_{t})\right\|^{2}\right]  \!\!& \leq \!\! \frac{2\tilde{L}(\cL(\bW_{0}, \bg_{0}) \!\!-\!\! \cL(\bW^{*}, \bg^{*}))}{\sqrt{T}} \\
		& + \frac{\sigma^{2}\tilde{L}}{2\sqrt{T}}\sum\limits_{i=1}^{m}\left(\frac{1}{\tilde{L}_{\bg}} + \frac{1}{\tilde{L}_{\bw^{(i)}}}\right),
		\end{aligned}
		\end{equation}
		by choosing $\eta_{\bw_{t}^{(i)}}=\frac{1}{\tilde{L}_{\bw^{(i)}}\sqrt{T}}$ and $\eta_{\bg_{t}}=\frac{1}{\tilde{L}_{\bg}\sqrt{T}}$, where $\tilde{L}, \tilde{L}_{\bw^{(i)}}$ and $\tilde{L}_{\bg}$ are defined in Theorem \ref{pro:GD learning rate}.
	\end{theorem}
	The two convergence rates of $\|\nabla_{\bW,\bg}\cL(\bW,\bg)\|$ match the optimal results with well tuned learning rates \citep{ghadimi2013stochastic}. Since $L_{ij}^{\bv\bv}\leq L_{ij}^{\bw\bw}; L_{i}^{\bv\bg}\leq L_{i}^{\bw\bg}$, the right hand side in \eqref{eq:rates of GD} is smaller than the one in \eqref{eq:rates of PSI-GD}. To see the results for SGD, under \eqref{eq:lipschitz}, the upper bound in \eqref{eq:rates of PSI-SGD} can be replaced by the smaller one of the upper bounds in \eqref{eq:rates of PSI-SGD} and \eqref{eq:rates of SGD}. Thus the convergence rate of PSI-SGD is also sharper than the one of SGD. 
	\par
	On the other hand, \eqref{eq:smoothness} and the discussion in the above section implies $\|\nabla_{\bW_{t},\bg_{t}} \cL(\bW_{t}, \bg_{t})\|^{2} \leq \|\nabla_{\bV_{t},\bg_{t}} \cL(\bV_{t}, \bg_{t})\|^{2}$. Thus, from \eqref{eq:smoothness}, one can verify the proposed algorithms accelerate training in a factor $\|\bw_{t}^{(i)}\|^{2}\geq 1$ compared with the versions on Euclidean space. In addition, $\|\bw_{t}^{(i)}\|^{2}$ keeps increasing to a bounded constant across training, thus the acceleration is more significant after a number of iterations. A detailed discussion to $\|\bw_{t}^{(i)}\|^{2}$ is in Appendix \ref{app:upper bound}.   
	\par
	The proofs of the two theorems are delegated to Appendix \ref{app:proof of pro1} and \ref{app:proof of pro2}. In the proof, we see that scaling the learning rate with $1/\tilde{L}_{\bw^{(i)}}$ and $1/\tilde{L}_{\bg}$ are respectively the optimal schedule of $\eta_{\bw_{t}^{(i)}}$ and $\eta_{\bg_{t}}$. Since smaller $\tilde{L}_{\bw^{(i)}}$ and $\tilde{L}_{\bg}$ corresponds with a smoother loss landscapes, it explains that $\cL(\bW, \bg)$ allows a larger learning rate in a smoother region to get the optimal convergence rate. 
	\par
	Now we are ready to illustrate the reason of PSI-GD and PSI-SGD's acceleration. Due to $\cL(\bW,\bg) = \cL(T_{\ba}(\bW),\bg)$,
	\begin{equation}\small\label{eq:hessian}
	\begin{aligned}
	\nabla^{2}_{\bw^{(i)}\bw^{(j)}}\cL(\bW,\bg) & = a_{i}a_{j}\nabla^{2}_{a_{i}\bw^{(i)}a_{j}\bw^{(j)}}\cL(T_{\ba}(\bW), \bg); \\
	\nabla^{2}_{\bw^{(i)}\bg}\cL(\bW,\bg) & = a_{i}\nabla^{2}_{a_{i}\bw^{(i)}\bg}\cL(T_{\ba}(\bW), \bg).
	\end{aligned}
	\end{equation}
	The fact shows that the smoothness of PSI parameters increasing with its scale. As larger $\ba$ in the right hand side of \eqref{eq:hessian} results in smaller $\|\nabla^{2}_{a_{i}\bw^{(i)}a_{j}\bw^{(j)}}\cL(T_{\ba}(\bW), \bg)\|_{2}$ and $\|\nabla^{2}_{a_{i}\bw^{(i)}\bg}\cL(T_{\ba}(\bW), \bg)\|_{2}$. Due to Lemma \ref{lem:increasing}, the increasing $\|\bw_{t}^{(i)}\|$ implies the loss landscape of PSI parameters becomes smoother across training. Thus, gradually increasing the learning rate to update PSI parameters can accelerate training. The proposed PSI-GD and PSI-SGD are proven to be vanilla GD and SGD with the increasing learning rate that is proportional to $\|\bw_{t}^{(i)}\|^{2}$, which interprets the acceleration.
	\par
	Finally, \eqref{eq:hessian} shows that our algorithms are optimal in the manner of leveraging the smoothness to accelerate training. The intuition is that the proposed methods have a more accurate estimation of local smoothness across training. We point out that $L_{ij}^{\bv\bv}\leq L_{ij}^{\bw\bw}; L_{i}^{\bv\bg}\leq L_{i}^{\bw\bg}$ in \eqref{eq:lipschitz} also gives a sharper convergence rate of PSI-GD and PSI-SGD.
	\section{Experiments}\label{sec5}
	\subsection{Improved Convergence Rate}
	\begin{figure}[t!]\centering
		\includegraphics[width=0.3\textwidth]{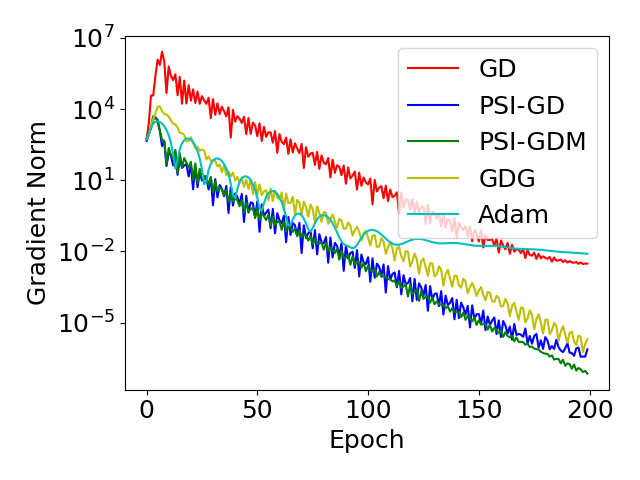}
		\caption{Convergence rates of gradient norm.}
		\label{fig:gradient_norm}
	\end{figure}
	We use a toy example to verify that our algorithms have improved convergence rate compared with other methods. 
	\paragraph{Data.} We sample 1000 training samples $\{\bx_{i}\}$ from 10-dimensional normal distribution with its label $y_{i} = \mu^{\top}\bx_{i} + \epsilon_{i}$ for $\epsilon_{i}\sim\cN(0, 1)$. 
	\paragraph{Setup.} We use a toy neural network $f(\bx) = \bw_{1}^{\top}\phi(W_{2}\bx)$ with $\bw_{1}\in\bbR^{100}$, $W_{2}\in\bbR^{100\times10}$, and the activation function $\phi(\cdot)$ is ELU function. The training loss is MES loss. We compare the convergence rates in terms of gradient norm of the proposed algorithms PSI-SGD and PSI-SGDM with three baselines algorithms. The benchmark optimization algorithm SGD with momentum (abbrev. SGD); adaptive learning rate algorithm Adam \citep{kingma2014adam}; manifold based algorithm SGDG \citep{cho2017riemannian} which is applied to the network with BN.
	\paragraph{Main Results.} The convergence rates of gradient norm are in Figure \ref{fig:gradient_norm}. As can be seen, the three manifold based methods PSI-SGD, PSI-SGDM, SGDG exhibit improved convergence rates compared other methods. However, we still observe that PSI-SGD and PSI-SGDM are slightly better than SGDG. More important is that in contrast to SGDG, our algorithms have proved convergence rate. The toy example verifies that our algorithms have improved convergence rates compared with other methods.  
	\subsection{Experiments on Real-world Dataset}
	In this section, we empirically study the proposed algorithms PSI-SGD and PSI-SGDM on real-world dataset. 
	\paragraph{Data.} We consider the image classification task on three benchmark datasets.\texttt{CIFAR10} and \texttt{CIFAR100} \citep{krizhevsky2009learning} are respectively colorful images with 50K training samples and 10K validation samples from 10 and 100 categories. \texttt{ImageNet} \citep{deng2009imagenet} are colorful images with 1M+ training samples from 1K object classes.
	\paragraph{Setup.} The model is a unified structure ResNet with various structures. As in the above section, we compare our methods with SGD, Adam, and SGDG. For PSI-SGD or PSI-SGDM, the PSI parameters are updated by the two algorithms, while the other parameters are updated by SGD.  
	\par
	We do not use the regularizer to the PSI weights e.g., $l_{2}$-regularizer in the loss function since it breaks the PSI property of PSI weights. More experiments conducted with regularizer are in Appendix \ref{app:experiment}. 
	\par
	For \texttt{CIFAR} we conduct 200 epochs of training for each algorithm. The learning rate starts from 0.1 and decays by a factor 0.2 at epochs 60, 120, and 160. For \texttt{ImageNet}, the training is conducted for 100 epochs, and the learning rate starts from 0.1 and decays by a factor 0.2 at epochs 30, 60, and 90. For the hyperparameters of baseline methods, we grid search the learning rates in the range of \{0.01, 0.1, 1.0\} and \{0.0001, 0.001, 0.01\} respectively for SGD and Adam, and the hyperparameters of SGDG follow the one of \citep{cho2017riemannian}. The other hyperparameters of all these methods are summarized in Appendix \ref{app:hyperparameters}.
	\par
	It worth noting that the PSI weights updated by PSI-SGD and PSI-SGDM may overflow after quite a number of iterations. Hence, we normalize the PSI weights $\bw^{(i)}$ if $\|\bw^{(i)}\|$ is larger than $10000$ during training. This operation will not change the output of the model due to the PSI property.
	\paragraph{Main Results.} 
	We report the test accuracy of each method. The results refers to Table  \ref{tab:resnet_without_regularizer} and Figure \ref{fig:result}. We have the following observations and conclusions from the experimental results. 
	\begin{enumerate}
		\item In Table \ref{tab:resnet_without_regularizer}, the model trained by manifold based methods i.e., SGDG, PSI-SGD and, PSI-SGDM generalize better in most cases. this is due to the manifold based algorithms can obviate quite a lot of local minima with poor generalization, since they have a unified optimization path for the parameters equivalent to each other.  
		\par
		Besides that, the proposed PSI-SGD and PSI-SGDM are significantly better than SGDG, and the performances of PSI-SGD and PSI-SGDM are comparable. We speculate this is due to the sharper convergence rate of PSI-SGD and PSI-SGDM allow them to find local minima in less number of iterations.  
		\item Figure \ref{fig:result} show that the PSI-SGD and PSI-SGDM converge faster than the three baselines after a certain number of iterations i.e., after 120 epochs of update. This justifies our theoretical results in Section \ref{sec:convergence results}, since the acceleration is linearly with  $\|\bw_{t}^{(i)}\|^{2}$ which is large after a while of training (Lemma \ref{lem:increasing}).
		\par
		We present the results of training error instead of gradient norm here because evaluating the gradient norm requires implementing back propagations on all training data which brings a great extra computational effort, especially for large scale dataset e.g., ImageNet.  
	\end{enumerate}
	We have one remark about the efficiency of our methods. They are simple and efficient compared with SGDG, since SGDG involves the operators like trigonometric functions in the update rule. For example, the elapsed time of ResNet56 with respect to one-epoch training of \texttt{CIFAR10} under SGD, Adam, SGDG, PSI-SGD, and PSI-SGDM are respectively 16.98, 18.1, 26.6, 18.2, and 18.5 seconds. All the experiments are conducted on a server with single NVIDIA V100 GPU.
	\section{Conclusion}
	In this paper, we fix the optimization ambiguity brought by the PSI property of the network with BN. Our scenario is built upon optimization on manifold by constructing a specific manifold and optimizing the PSI weights of the network with BN in it. The developed gradient-based algorithms on PSI manifold are shown to have a well-defined optimization path with respect to positively equivalent rescaling. 
	\par
	We also give the convergence rates of the proposed methods. Besides that, we theoretically justify that PSI-GD and PSI-SGD accelerate training by a clever schedule of adaptive learning rate. 
	\par
	Finally, we conduct various experiments to show that the proposed methods have better performance in the generalization and efficiency compared with the other three baselines.   
\bibliography{reference} 
\bibliographystyle{plainnat}
\newpage
\appendix
\input{appendix}
\end{document}

%% file: appendix.tex
\onecolumn
\section{Related Definitions of Optimization on Manifold}\label{app: definitions of manifold}
We give a brief introduction to optimization on manifold in this section. A summary of this topic can be referred to \citep{absil2009optimization}. This paper focuses on the matrix manifold which is a subspace of Euclidean space $\mathbb{R}^{n}$. We start with the definition of matrix manifold.
\begin{definition}[Matrix manifold]
	A Matrix Manifold $\mathcal{M}$ is a subset of Euclidean space $\mathbb{R}^{n}=\mathbb{R}^{m\times p}$ with any $\bx\in\mathcal{M}$ has a neighborhood $U_{\bx}$ of $\bx$ such that $U_{\bx}$ is homeomorphic to a Euclidean space. In addition, for a given equivalent relation $\sim$ of elements in the manifold $\overline{\mathcal{M}}$, we use $[\bx]$ to denote set $\{\by\in\overline{\mathcal{M}}:\by\sim \bx\}$. Then $\mathcal{M} = \overline{\mathcal{M}}/\sim=\{[\bx]:\bx\in\overline{\mathcal{M}}\}$ is a quotient manifold.\footnote{Please notice we use $\mathcal{M}$ to represent quotient manifold while $\overline{\mathcal{M}}$ is the original manifold equipped with equivalent relationship.}
	\label{def:manifold}
\end{definition}
Various spaces can be categorized into matrix manifold, e.g. Euclidean space, unit ball. For a manifold $\mathcal{M}$ with $\bx\in\mathcal{M}$, there is a tangent space $\cT_{\bx}\mathcal{M}$ which gives the tangential direction of $\bx\in\mathcal{M}$. For matrix manifold, tangent space $\cT_{\bx}\mathcal{M}$ must be a linear space with finite dimension \citep{absil2009optimization}, it specifies the moving direction of a point in the manifold. Since manifold can be non-linear space, the movement of a point $\bx\in\mathcal{M}$ along a specific direction $\bXi\in \cT_{\bx}\mathcal{M}$ is decided by the retraction function $R_{\bx}(\cdot):\cT_{\bx}\mathcal{M}:\rightarrow\mathcal{M}$, which is defined as follows.
\begin{definition}[Retraction]
	A retraction on a manifold $\mathcal{M}$ is a smooth mapping $R_{\bx}(\cdot):\cT_{\bx}\mathcal{M}:\rightarrow\mathcal{M}$ satisfies:
	\par
	(1) $R_{\bx}(0_{\bx})=\bx$, where $0_{\bx}$ denotes the zero element of $\cT_{\bx}\mathcal{M}$.
	\par
	(2) $R_{\bx}(\cdot)$ satisfies
	\[\lim_{t\to\infty}\frac{R_{\bx}(0_{\bx}+t\bXi)-R_{\bx}(0_{\bx})}{t}=\bXi\]
	for any $\bXi\in \cT_{\mathcal{M}}$.
\end{definition}
For instance, the retraction function of $\mathbb{R}^{n}$ can be defined as $R_{\bx}(\bXi)=\bx+\bXi$; and for unit ball it can be written as $R_{\bx}(\bXi)=\frac{\bx+\bXi}{\|\bx+\bXi\|}$ where $\bXi\in \cT_{\bx}\mathcal{M}$.
\par
To obtain the gradient based optimization algorithms on manifold, we should give the definition of Riemannian gradient in manifold. The following Riemannian metric derives the definition of corresponded gradient. 
\begin{definition}[Riemannian metric]
	For matrix manifold, Reimannian metric is an inner product $\langle\cdot,\cdot \rangle_{\bx}$ of tangent space $\cT_{\bx}\mathcal{M}$ for any $\bx\in\mathcal{M}$. In addition, given a set of coordinate basis vector $\{E_{i}\}$ of $\cT_{\bx}\mathcal{M}$, then for any $\bXi_{\bx}$ and $\bzeta_{\bx}$, we have
	\begin{equation}\small
	\langle\bXi_{\bx},\bzeta_{\bx} \rangle_{\bx} = \hat{\bXi}_{\bx}^{T}G_{\bx}\hat{\bzeta}_{\bx}.
	\label{eq:Riemannian metric}
	\end{equation}
	Here $\hat{\bXi}_{\bx}, \hat{\bzeta}_{\bx}$ are respective coordinates of $\bXi_{\bx}$ and $\bzeta_{\bx}$ under the coordinate basis vector $\{E_{i}\}$, and $G_{\bx}=(g)_{ij}$ where $g_{ij}=\langle E_{i},E_{j} \rangle_{\bx}$. 
\end{definition}
Based on the Riemannian metric, the Riemannian gradient is defined as follows.
\begin{definition}[Riemannian gradient]\label{def:riemannian gradient}
	The Riemannian gradient of function $f(\cdot)$ defined in an open set containing manifold $\mathcal{M}$ at $\bx\in\mathcal{M}$ is the tangent vector ${\rm grad}f(\bx)$ belongs to $\cT_{\bx}\mathcal{M}$ satisfies that
	\[\lim_{t\to\infty}\frac{f(\bx+t\bXi)-f(\bx)}{t}=\langle\bXi_{\bx},{\rm grad}f(\bx) \rangle_{\bx}\]
	for any $\bXi\in \cT_{\bx}\mathcal{M}$. 
\end{definition}
In fact, if function $f(\cdot)$ has gradient at point $\bx$, then we see
\begin{equation}\small
\begin{aligned}
\langle\bXi_{\bx},\nabla f(\bx)\rangle &= \lim_{t\to\infty}\frac{f(\bx+t\bXi_{\bx})-f(\bx)}{t} = \langle\bXi_{\bx},{\rm grad}f(\bx) \rangle_{\bx} \\
&= \bXi_{\bx}^{T}G_{\bx}{\rm grad}f(\bx).
\label{eq:riemannian gradient}
\end{aligned}
\end{equation}
Hence, we have ${\rm grad}f(\bx)=G_{\bx}^{-1}\nabla f(\bx)$ by the arbitrariness of $\bXi_{\bx}$. 
\par
\cite{boumal2019global} proves that for function $f(\cdot)$ defined on manifold $\mathcal{M}$ with Lipschitz Riemannnian gradient\footnote{Which means there exist a positive number $L_{g}$ satisfy $\|{\rm grad}f(\bx)-{\rm grad}f(\by)\|\leq L_{g}\|\bx-\by\|$ for any $\bx,\by\in\mathcal{M}$}, gradient descent in manifold
\begin{equation}\small
\bx_{t+1}=R_{\bx_{t}}(-{\rm grad}f(\bx_{t}))
\label{eq:gradient descent}
\end{equation}
can convergence to a local minimum of $f(\cdot)$. 
\section{Proofs in Section \ref{sec:PSI manifold}}
\subsection{Proof of Proposition \ref{pro:riemannian metric}}\label{app:proof of riemannian metric}
\begin{proof}
	It is easily to verify that the function $\langle\cdot,\cdot\rangle_{\bW}$ is an inner product for every $\bW\in\mathcal{M}$. We notice that $\bXi\in \cT_{\bW}\mathcal{M}$ is a representation $\bar{\bXi}_{\overline{\bW}}\in \cT_{\overline{\bW}}\overline{\mathcal{M}}$ for some $\overline{\bW}\in \pi^{-1}(\bW)$. Similar to Example 3.5.4 in \citep{absil2009optimization}, we see different representation of $\bXi\in \cT_{\bW}\mathcal{M}$ satisfies that 
	\begin{equation}\small
	\bar{\bXi}_{T_{\ba}(\overline{\bW})} = T_{\ba}(\bar{\bXi}_{\overline{\bW}}).
	\end{equation}
	Then, combining equation \eqref{eq:riemannian metric}, we can conclude that the Riemannian metric is well-defined to the choice of horizontal lift $\bar{\bXi}_{\overline{\bW}}$. Hence, we get the conclusion.
\end{proof}
\subsection{Proof of Proposition \ref{pro:retr}}\label{app:proof of reatraction}
We need a Lemma to give the proof.
\begin{lemma}[Proposition 4.1.3 in \citep{absil2009optimization}]
	Let $\mathcal{M}=\overline{\mathcal{M}}/\sim$ be a quotient manifold and $\bar{R}$ be a retraction on $\overline{\mathcal{M}}$ such that for all $x\in\mathcal{M}$ any two $\bar{\bx}_{a},\bar{\bx}_{b}\in\pi^{-1}(\bx)$ and $\bxi\in \cT_{\bx}\mathcal{M}$,
	\begin{equation}\small
	\pi(\bar{R}_{\bar{\bx}_{a}}(\bar{\bxi}_{\bar{\bx}_{a}})) = \pi(\bar{R}_{\bar{\bx}_{b}}(\bar{\bxi}_{\bar{\bx}_{b}})).
	\end{equation} 
	Here $\pi(\cdot)$ is a canonical projection from $\overline{\mathcal{M}}$ to $\mathcal{M}$ satisfies that $\pi(\bx)=\pi(\by)$ if and only if $\bx\sim \by$. $\bar{\bxi}_{\bar{\bx}_{a}}$, $\bar{\bxi}_{\bar{\bx}_{b}}$ are respectively tangent vector in $\cT_{\bar{\bx}_{a}}\overline{\mathcal{M}}$ and $\cT_{\bar{\bx}_{b}}\overline{\mathcal{M}}$. $\bar{R}_{\cdot}(\cdot)$ is a retraction in $\overline{\mathcal{M}}$. Then 
	\begin{equation}\small
	R_{\bx}(\bxi) = \pi(\bar{R}_{\bar{\bx}}(\bar{\bxi}_{\bar{\bx}}))
	\end{equation}
	defines a retraction on quotient manifold $\mathcal{M}$.
	\label{lem:retraction}
\end{lemma}
With this Lemma, we now provide the proof of Proposition \ref{pro:retr}. 
\begin{proof}[Proof of Proposition \ref{pro:retr}]
	According to the lemma \ref{lem:retraction}, we only need to verify the conditions. First we notice that $\bXi\in \cT_{W}\mathcal{M}$ is a representation $\bar{\bXi}_{\overline{\bW}}\in \cT_{\overline{\bW}}\overline{\mathcal{M}}$ for some $\overline{\bW}\in \pi^{-1}(\bW)$. Similar to the Example 3.5.4 in \citep{absil2009optimization}, different representation of $\bXi\in \cT_{\bW}\mathcal{M}$ satisfies
	\begin{equation}\small
	\bar{\bXi}_{T_{\ba}(\overline{\bW})} = T_{\ba}(\bar{\bXi}_{\overline{\bW}}).
	\end{equation}
	Hence we can choose
	\begin{equation}\small
	\bar{R}_{\overline{\bW}}(\bar{\bXi}_{\overline{\bW}}) = \overline{\bW} + \bar{\bXi}_{\overline{\bW}}.
	\end{equation}
	Thus $\bar{R}_{\overline{\bW}}(\cdot)$ is a retraction in $\overline{\mathcal{M}}$. Then, for any canonical projection $\pi(\cdot)$ there exist $\pi(\bar{R}_{\overline{\bW}}(\bar{\bXi}_{\overline{\bW}})) = \pi(\bar{R}_{\cT_{\ba}(\overline{\bW})}(\bar{\bXi}_{\cT_{\ba}(\overline{\bW})}))$. Hence, we get the conclusion.
\end{proof}
\section{Proofs in Section \ref{sec:optimization on PSI manifold}}
\subsection{Proof of Theorem \ref{thm:thm2}}\label{app:proof of unified}
\begin{proof}
	Since
	\begin{equation}
	\small
	\nabla_{\bw^{(i)}}\cL(\bW,\bg) = a_{i}\nabla_{a_{i}\bw^{(i)}}\cL(T_{\ba}(\bW), \bg)
	\end{equation}
	due to the PSI property. Then, combining the update rule of PSI-SGDM \eqref{eq:riemannian gradient descent with momentum}, we have  
	\begin{equation}\small
	\begin{aligned}
	\bU_{0}\xrightarrow{\text{update}} \bU_{1}; \qquad T_{\ba}(\bU_{0})& \xrightarrow{\text{update}} T_{\ba}(\bU_{1})\\
	\bW_{0}\xrightarrow{\text{update}} \bW_{1}; \qquad T_{\ba}(\bW_{0})& = \hat{\bW}_{0} \xrightarrow{\text{update}} \hat{\bW}_{1} = T_{\ba}(\bW_{1}).
	\end{aligned}
	\end{equation} 
	By induction, we get the conclusion.
\end{proof}
\section{Proofs in Section \ref{sec:convergence results}}
\subsection{Proof of Theorem \ref{pro:GD learning rate}}\label{app:proof of pro1}
\begin{proof}
	By Taylor's expansion, and condition of Lipschitz gradient
	\begin{equation}
	\small
	\begin{aligned}
	\cL(\bW_{t + 1},\bg_{t + 1}) &- \cL(\bW_{t}, \bg_{t}) = \int_{0}^{1} \left\langle \nabla_{\bW_{t},\bg_{t}}\cL(\bW_{t} + s\Delta\bW_{t}, \bg_{t} + s\Delta\bg_{t}), (\Delta\bW_{t}, \Delta\bg_{t})\right\rangle ds\\
	& = \left\langle \nabla_{\bW,\bg}\cL(\bW_{t}, \bg_{t}), (\bW_{t}, \bg_{t})\right\rangle\\
	& + \int_{0}^{1} \left\langle \nabla_{\bW_{t},\bg_{t}}\cL(\bW_{t} + s\Delta\bW_{t}, \bg_{t} + s\Delta\bg_{t}) - \nabla_{\bW,\bg}\cL(\bW_{t}, \bg_{t}), (\Delta\bW_{t}, \Delta\bg_{t})\right\rangle ds\\
	& \leq \left\langle \nabla_{\bW,\bg}\cL(\bW_{t}, \bg_{t}), (\bW_{t}, \bg_{t})\right\rangle + \frac{1}{2}\sum\limits_{i,j=1}^{m} L_{ij}^{\bw\bw}\left\|\Delta\bw_{t}^{(i)}\right\|\left\|\Delta\bw_{t}^{(j)}\right\|\\
	& + \sum\limits_{i=1}^{m}L_{i}^{\bw\bg} \left\|\Delta\bw_{t}^{(i)}\right\|\left\|\Delta\bg_{t}\right\| + \frac{mL^{\bg\bg}}{2}\|\left\|\Delta\bg_{t}\right\|
	\end{aligned}
	\end{equation}
	Thus, 
	\begin{equation}
	\small
	\begin{aligned}
	\cL(\bW_{t + 1},\bg_{t + 1}) - \cL(\bW_{t}, \bg_{t}) &\leq -\sum\limits_{i=1}^{m}\left(\eta_{\bw_{t}^{(i)}}\left\|\nabla_{\bw_{t}^{(i)}}\cL(\bW_{t}, \bg_{t})\right\|^{2} + \eta_{\bg_{t}}\left\|\nabla_{\bg_{t}}\cL(\bW_{t},\bg_{t})\right\|^{2}\right) \\
	& + \frac{1}{2}\sum\limits_{i,j=1}^{m}L_{ij}^{\bw\bw} \left\|\eta_{\bw_{t}^{(i)}}\nabla_{\bw_{t}^{(i)}}\cL(\bW_{t}, \bg_{t})\right\|\left\|\eta_{\bw_{t}^{(j)}}\nabla_{\bw_{t}^{(j)}}\cL(\bW_{t}, \bg_{t})\right\|\\
	& + \sum\limits_{i=1}^{m}L_{i}^{\bw\bg}\left\|\eta_{\bw_{t}^{(i)}}\nabla_{\bw_{t}^{(i)}}\cL(\bW_{t}, \bg_{t})\right\|\left\|\eta_{\bg_{t}} \nabla_{\bg_{t}}\cL(\bW_{t},\bg_{t})\right\|\\
	& + \frac{mL^{\bg\bg}}{2}\left\|\eta_{\bg_{t}}\nabla_{\bg_{t}}\cL(\bW_{t},\bg_{t})\right\|^{2}.
	\end{aligned}
	\end{equation}
	Then, by Young's inequality, 
	\begin{equation}
	\small
	\begin{aligned}
	\left\|\eta_{\bw_{t}^{(i)}}\nabla_{\bw_{t}^{(i)}}\cL(\bW_{t}, \bg_{t})\right\|\left\|\eta_{\bw_{t}^{(j)}}\nabla_{\bw_{t}^{(j)}}\cL(\bW_{t}, \bg_{t})\right\| & \leq \frac{1}{2}\left(\left\|\eta_{\bw_{t}^{(i)}}\nabla_{\bw_{t}^{(i)}}\cL(\bW_{t}, \bg_{t})\right\|^{2} + \left\|\eta_{\bw_{t}^{(j)}}\nabla_{\bw_{t}^{(j)}}\cL(\bW_{t}, \bg_{t})\right\|^{2}\right) \\
	\left\|\eta_{\bw_{t}^{(i)}}\nabla_{\bw_{t}^{(i)}}\cL(\bW_{t}, \bg_{t})\right\|\left\|\eta_{\bg_{t}}\nabla_{\bg_{t}}\cL(\bW_{t}, \bg_{t})\right\| & \leq \frac{1}{2}\left(\left\|\eta_{\bw_{t}^{(i)}}\nabla_{\bw_{t}^{(i)}}\cL(\bW_{t}, \bg_{t})\right\|^{2} + \left\|\eta_{\bg_{t}}\nabla_{\bg_{t}}\cL(\bW_{t}, \bg_{t})\right\|^{2}\right).
	\end{aligned}
	\end{equation}
	Then, we see
	\begin{equation}\label{eq: gradient bound}
	\small
	\begin{aligned}
	\cL(\bW_{t + 1},\bg_{t + 1}) &- \cL(\bW_{t}, \bg_{t}) \leq \sum\limits_{i=1}^{m}\left(-\eta_{\bw_{t}^{(i)}} + \frac{\tilde{L}_{\bw^{(i)}}}{2}\eta_{\bw_{t}^{(i)}}^{2}\right)\left\|\nabla_{\bw_{t}^{(i)}}\cL(\bW_{t}, \bg_{t})\right\|^{2}\\
	& + \left(-\eta_{\bg_{t}} + \frac{\tilde{L}_{\bg}}{2}\eta_{\bg_{t}}^{2}\right)\|\nabla_{\bg_{t}}\cL(\bW_{t},\bg_{t})\|^{2}\\
	& \leq \left\|\nabla \cL_{\bW_{t},\bg_{t}}(\bW_{t}, \bg_{t})\right\|^{2}\max_{1\leq i\leq m}\left\{\left(-\eta_{\bw_{t}^{(i)}} +  \frac{\tilde{L}_{\bw^{(i)}}}{2}\eta_{\bw_{t}^{(i)}}^{2}\right)\right\}\wedge \left(-\eta_{\bg_{t}} + \frac{\tilde{L}_{\bg}}{2}\eta_{\bg_{t}}^{2}\right)
	\end{aligned}
	\end{equation}
	where
	\begin{equation}
	\small
	\tilde{L}_{\bw^{(i)}} = L_{i}^{\bw\bg} + \sum\limits_{j=1}^{m}L_{ij}^{\bw\bw};\qquad \tilde{L}_{\bg} = mL^{\bg\bg} + \sum\limits_{i=1}^{m}L_{i}^{\bw\bg}.
	\end{equation}
	Summing over equation \eqref{eq: gradient bound} from $0$ to $T-1$, let $0\leq \eta_{\bw_{t}^{(i)}}\leq 2/\tilde{L}_{\bw^{(i)}}$, we get
	\begin{equation}
	\small
	\begin{aligned}
	\cL(\bw_{T}, \bg_{T}) &- \cL(\bW_{0}, \bg_{0})\\
	& \leq \min_{0\leq t< T}\left\|\nabla \cL_{\bW_{t},\bg_{t}}(\bW_{t}, \bg_{t})\right\|^{2}\sum\limits_{t=0}^{T - 1}\max_{1\leq i\leq m}\left\{\left(-\eta_{\bw_{t}^{(i)}} + \frac{\tilde{L}_{\bw^{(i)}}}{2}\eta_{\bw_{t}^{(i)}}^{2}\right)\right\}\wedge \left(-\eta_{\bg_{t}} + \frac{\tilde{L}_{\bg}}{2}\eta_{\bg_{t}}^{2}\right).
	\end{aligned}
	\end{equation}
	We see the best selection of $\eta_{\bw_{t}^{(i)}}$ and $\eta_{\bg_{t}}$ are respectively $1/\tilde{L}_{\bw^{(i)}}$ and $1/\tilde{L}_{\bg}$. By the value of $\eta_{\bw_{t}^{(i)}}$ and $\eta_{\bg_{t}}$, we get 
	\begin{equation}
	\small
	\min_{0\leq t< T}\left\|\nabla \cL_{\bW_{t},\bg_{t}}(\bW_{t}, \bg_{t})\right\|^{2} \max\left\{\frac{1}{2\tilde{L}_{\bw^{(i)}}}, \cdots, \frac{1}{2\tilde{L}_{\bw^{(m)}}}, \frac{1}{2\tilde{L}_{\bg}}\right\} \leq \frac{\cL(\bW_{0}, \bg_{0}) - \cL(\bW^{*}, \bg^{*})}{T}.
	\end{equation}
	Then we conclude the proof. 
\end{proof}
\subsection{Proof of Theorem \ref{pro:SGD learning rate}}\label{app:proof of pro2}
\begin{proof}
	By conditions of Lipschitz gradient, we have 
	\begin{equation}\label{eq:bound difference}
	\small
	\begin{aligned}
	\cL(\bW_{t + 1},\bg_{t + 1}) - \cL(\bW_{t}, \bg_{t}) &\leq -\sum\limits_{i=1}^{m}\left(\eta_{\bw_{t}^{(i)}}\left\|\nabla_{\bw_{t}^{(i)}}\cL(\bW_{t}, \bg_{t})\right\|^{2} + \eta_{\bg_{t}}\left\|\nabla_{\bg_{t}}\cL(\bW_{t},\bg_{t})\right\|^{2}\right) \\
	& + \frac{1}{2}\sum\limits_{i,j=1}^{m}L_{ij}^{\bw\bw} \left\|\eta_{\bw_{t}^{(i)}}\cG_{\bw_{t}^{(i)}}(\bW_{t},\bg_{t})\right\|\left\|\eta_{\bw_{t}^{(j)}}\cG_{\bw_{t}^{(j)}}(\bW_{t},\bg_{t})\right\|\\
	& + \sum\limits_{i=1}^{m}L_{i}^{\bw\bg}\left\|\eta_{\bw_{t}^{(i)}}\cG_{\bw_{t}^{(i)}}(\bW_{t},\bg_{t})\right\|\left\|\eta_{\bg_{t}} \cG_{\bg_{t}}(\bW_{t},\bg_{t})\right\|\\
	& + \frac{mL^{\bg\bg}}{2}\left\|\eta_{\bg_{t}}\cG_{\bg_{t}}(\bW_{t},\bg_{t})\right\|^{2}.
	\end{aligned}
	\end{equation}
	By Young's inequality, taking expectation conditional in $(\bW_{t},\bg_{t})$, we have 
	\begin{equation}\label{eq:bound1}
	\small
	\begin{aligned}
	&\mE\left[\left\|\eta_{\bw_{t}^{(i)}}\cG_{\bw_{t}^{(i)}}(\bW_{t},\bg_{t})\right\|\left\|\eta_{\bw_{t}^{(j)}}\cG_{\bw_{t}^{(j)}}(\bW_{t},\bg_{t})\right\|\right] \\
	& \leq \frac{1}{2}\left(\mE\left[\left\|\eta_{\bw_{t}^{(i)}}\cG_{\bw_{t}^{(i)}}(\bW_{t},\bg_{t})\right\|^{2}\right] + \mE\left[\left\|\eta_{\bw_{t}^{(j)}}\cG_{\bw_{t}^{(j)}}(\bW_{t},\bg_{t})\right\|^{2}\right]\right)\\
	& = \frac{1}{2}\eta_{\bw_{t}^{(i)}}^{2}\mE\left[\left\|\cG_{\bw_{t}^{(i)}}(\bW_{t},\bg_{t}) - \nabla_{\bw_{t}^{(i)}}\cL(\bW_{t}, \bg_{t}) + \nabla_{\bw_{t}^{(i)}}\cL(\bW_{t}, \bg_{t})\right\|^{2}\right] \\
	& + \frac{1}{2}\eta_{\bw_{t}^{(j)}}^{2}\mE\left[\left\|\cG_{\bw_{t}^{(j)}}(\bW_{t},\bg_{t}) - \nabla_{\bw_{t}^{(j)}}\cL(\bW_{t}, \bg_{t}) + \nabla_{\bw_{t}^{(j)}}\cL(\bW_{t}, \bg_{t})\right\|^{2}\right] \\
	& \overset{a}{=} \frac{1}{2}\eta_{\bw_{t}^{(i)}}^{2} \mE\left[\left\|\cG_{\bw_{t}^{(i)}}(\bW_{t},\bg_{t}) - \nabla_{\bw_{t}^{(i)}}\cL(\bW_{t}, \bg_{t})\right\|^{2} + \left\|\nabla_{\bw_{t}^{(i)}}\cL(\bW_{t}, \bg_{t})\right\|^{2}\right] \\
	& + \frac{1}{2}\eta_{\bw_{t}^{(j)}}^{2}\mE\left[\left\|\cG_{\bw_{t}^{(j)}}(\bW_{t},\bg_{t}) - \nabla_{\bw_{t}^{(j)}}\cL(\bW_{t}, \bg_{t})\right\|^{2} + \left\|\nabla_{\bw_{t}^{(j)}}\cL(\bW_{t}, \bg_{t})\right\|^{2}\right] \\
	& \leq \frac{1}{2}\eta_{\bw_{t}^{(i)}}^{2}\left(\sigma^{2} + \mE\left[\left\|\nabla_{\bw_{t}^{(i)}}\cL(\bW_{t}, \bg_{t})\right\|^{2}\right]\right) + \frac{1}{2}\eta_{\bw_{t}^{(j)}}^{2}\left(\sigma^{2} + \mE\left[\left\|\nabla_{\bw_{t}^{(j)}}\cL(\bW_{t}, \bg_{t})\right\|^{2}\right]\right),
	\end{aligned}
	\end{equation}
	where $a$ is due to equation \eqref{eq:bounded variance}. Similarly, we have
	\begin{equation}\label{eq:bound2}
	\small
	\begin{aligned}
	\mE\left[\left\|\eta_{\bw_{t}^{(i)}}\cG_{\bw_{t}^{(i)}}(\bW_{t},\bg_{t})\right\|\left\|\eta_{\bg_{t}} \cG_{\bg_{t}}(\bW_{t},\bg_{t})\right\|\right] &\leq \frac{1}{2}\eta_{\bw_{t}^{(i)}}^{2}\left(\sigma^{2} + \mE\left[\left\|\nabla_{\bw_{t}^{(i)}}\cL(\bW_{t}, \bg_{t})\right\|^{2}\right]\right)\\
	& + \frac{1}{2}\eta_{\bg_{t}}^{2}\left(\sigma^{2} + \mE\left[\left\|\nabla_{\bg_{t}}\cL(\bW_{t}, \bg_{t})\right\|^{2}\right]\right),
	\end{aligned}
	\end{equation}
	and 
	\begin{equation}\label{eq:bound3}
	\small
	\begin{aligned}
	\mE\left[\left\|\eta_{\bg_{t}}\cG_{\bg_{t}}(\bW_{t},\bg_{t})\right\|^{2}\right] &\leq \frac{1}{2}\eta_{\bg_{t}}^{2}\left(\sigma^{2} + \mE\left[\left\|\nabla_{\bg_{t}}\cL(\bW_{t}, \bg_{t})\right\|^{2}\right]\right).
	\end{aligned}
	\end{equation}
	Fixing the randomness of $\bW_{t},\bg_{t}$, plugging equation \eqref{eq:bound1}, \eqref{eq:bound2} and \eqref{eq:bound3} into \eqref{eq:bound difference}, then taking expectation, we get 
	\begin{equation}\small\label{eq:bound gradient}
	\begin{aligned}
	\mE\left[\cL(\bW_{t + 1},\bg_{t + 1})\right] &- \cL(\bW_{t}, \bg_{t}) \leq \sum\limits_{i=1}^{m}\left(-\eta_{\bw_{t}^{(i)}} + \frac{\tilde{L}_{\bw^{(i)}}}{2}\eta_{\bw_{t}^{(i)}}^{2}\right)\mE\left[\left\|\nabla_{\bw_{t}^{(i)}}\cL(\bW_{t}, \bg_{t})\right\|^{2}\right]\\
	& + \left(-\eta_{\bg_{t}} + \frac{\tilde{L}_{\bg}}{2}\eta_{\bg_{t}}^{2}\right)\mE\left[\|\nabla_{\bg_{t}}\cL(\bW_{t},\bg_{t})\|^{2}\right] + \frac{\sigma^{2}}{2}\left(\eta_{\bg_{t}}^{2}\tilde{L}_{\bg} + \sum\limits_{i=1}^{m}\eta_{\bw_{t}^{(i)}}^{2}\tilde{L}_{\bw^{(i)}} \right)\\
	& \overset{a}{\leq} -\frac{1}{\sqrt{T}}\sum\limits_{i=1}^{m}\frac{1}{\tilde{L}_{\bw^{(i)}}}\mE\left[\left\|\nabla_{\bw_{t}^{(i)}}\cL(\bW_{t}, \bg_{t})\right\|^{2}\right] - \frac{1}{\sqrt{T}\tilde{L}_{\bg}}\mE\left[\|\nabla_{\bg_{t}}\cL(\bW_{t},\bg_{t})\|^{2}\right]\\
	& + \frac{\sigma^{2}}{2T}\left(\frac{1}{\tilde{L}_{\bg}} + \sum\limits_{i=1}^{m}\frac{1}{\tilde{L}_{\bw^{(i)}}}\right)\\
	& \leq -\frac{1}{\sqrt{T}\tilde{L}}\mE\left[\left\|\nabla_{\bW_{t}, \bg_{t}}\cL(\bW_{t}, \bg_{t})\right\|^{2}\right] + \frac{\sigma^{2}}{2T}\left(\frac{1}{\tilde{L}_{\bg}} + \sum\limits_{i=1}^{m}\frac{1}{\tilde{L}_{\bw^{(i)}}}\right),
	\end{aligned}
	\end{equation}
	where $a$ is due to the value of $\eta_{\bw_{t}^{(i)}}$ and $\eta_{\bg_{t}}$, $\tilde{L}$ is $\max\{\tilde{L}_{\bw^{(1)}}, \cdots, \tilde{L}_{\bw^{(m)}}, \tilde{L}_{\bg}\}$. Summing over from $0$ to $T-1$ of equation \eqref{eq:bound gradient}, we get
	\begin{equation}
	\small
	\begin{aligned}
	\sum\limits_{t=0}^{T - 1}\frac{1}{\sqrt{T}\tilde{L}}\mE\left[\left\|\nabla_{\bW_{t}, \bg_{t}}\cL(\bW_{t}, \bg_{t})\right\|^{2}\right] \leq \cL(\bW_{0}, \bg_{0}) - \mE\left[\cL(\bW_{T}, \bg_{T})\right] + \frac{\sigma^{2}}{2}\sum\limits_{i=1}^{m}\left(\frac{1}{\tilde{L}_{\bg}} + \frac{1}{\tilde{L}_{\bw^{(i)}}}\right)\sqrt{T}
	\end{aligned}
	\end{equation}
	Thus, we have
	\begin{equation}
	\small
	\begin{aligned}
	\left(\frac{\sqrt{T}}{\tilde{L}}\right)\min_{0\leq t< T}\mE\left[\left\|\nabla_{\bW_{t}, \bg_{t}}\cL(\bW_{t}, \bg_{t})\right\|^{2}\right] \leq \cL(\bW_{0}, \bg_{0}) - \cL(\bW^{*}, \bg^{*})
	+ \frac{\sigma^{2}}{2}\sum\limits_{i=1}^{m}\left(\frac{1}{\tilde{L}_{\bg}} + \frac{1}{\tilde{L}_{\bw^{(i)}}}\right)
	\end{aligned}
	\end{equation}
	Then we get the conclusion. 
\end{proof}   
\subsection{Proof of Theorem \ref{thm:PSI-GD}}\label{app:proof of PSI-GD}
\begin{proof}
	By assumption \eqref{eq:lipschitz on manifold} and PSI property, we have
	\begin{equation}
	\small
	\begin{aligned}
	\cL(\bW_{t + 1},\bg_{t + 1}) - \cL(\bW_{t}, \bg_{t}) &\leq -\sum\limits_{i=1}^{m}\left(\eta_{\bv_{t}^{(i)}}\left\|\nabla_{\bv_{t}^{(i)}}\cL(\bV_{t}, \bg_{t})\right\|^{2} + \eta_{\bg_{t}}\left\|\nabla_{\bg_{t}}\cL(\bV_{t},\bg_{t})\right\|^{2}\right) \\
	& + \frac{1}{2}\sum\limits_{i,j=1}^{m}L_{ij}^{\bv\bv} \left\|\eta_{\bv_{t}^{(i)}}\nabla_{\bv_{t}^{(i)}}\cL(\bV_{t}, \bg_{t})\right\|\left\|\eta_{\bv_{t}^{(j)}}\nabla_{\bv_{t}^{(j)}}\cL(\bV_{t}, \bg_{t})\right\|\\
	& + \sum\limits_{i=1}^{m}L_{i}^{\bv\bg}\left\|\eta_{\bv_{t}^{(i)}}\nabla_{\bv_{t}^{(i)}}\cL(\bV_{t}, \bg_{t})\right\|\left\|\eta_{\bg_{t}} \nabla_{\bg_{t}}\cL(\bV_{t},\bg_{t})\right\|\\
	& + \frac{mL^{\bg\bg}}{2}\left\|\eta_{\bg_{t}}\nabla_{\bg_{t}}\cL(\bV_{t},\bg_{t})\right\|^{2}.
	\end{aligned}
	\end{equation}
	Here we use one fact that
	\begin{equation}\small
	\begin{aligned}
	\|\bw^{(i)}\|\|\bw^{(j)}\|\nabla^{2}_{\bw^{(i)}\bw^{(j)}}\cL(\bW,\bg) &= \nabla^{2}_{\bv^{(i)}\bv^{(j)}}\cL(\bV,\bg);\\
	\|\bw^{(i)}\|\nabla^{2}_{\bw^{(i)}\bg}\cL(\bW,\bg) &= \nabla^{2}_{\bv^{(i)}\bg}\cL(\bV,\bg).
	\end{aligned}
	\end{equation}
	Then follow the proof of Theorem \ref{pro:GD learning rate}, we get the conclusion. 
\end{proof}	
\subsection{Proof of Theorem \ref{thm:PSI-SGD}}\label{app:proof of PSI-SGD}
\begin{proof}
	By noticing that 
	\begin{equation}
	\small
	\begin{aligned}
	\cL(\bW_{t + 1},\bg_{t + 1}) - \cL(\bW_{t}, \bg_{t}) &\leq -\sum\limits_{i=1}^{m}\left(\eta_{\bv_{t}^{(i)}}\left\|\nabla_{\bv_{t}^{(i)}}\cL(\bV_{t}, \bg_{t})\right\|^{2} + \eta_{\bg_{t}}\left\|\nabla_{\bg_{t}}\cL(\bV_{t},\bg_{t})\right\|^{2}\right) \\
	& + \frac{1}{2}\sum\limits_{i,j=1}^{m}L_{ij}^{\bv\bv} \left\|\eta_{\bv_{t}^{(i)}}\cG_{\bv_{t}^{(i)}}(\bV_{t},\bg_{t})\right\|\left\|\eta_{\bv_{t}^{(j)}}\cG_{\bv_{t}^{(j)}}(\bV_{t},\bg_{t})\right\|\\
	& + \sum\limits_{i=1}^{m}L_{i}^{\bv\bg}\left\|\eta_{\bv_{t}^{(i)}}\cG_{\bv_{t}^{(i)}}(\bV_{t},\bg_{t})\right\|\left\|\eta_{\bg_{t}} \cG_{\bg_{t}}(\bV_{t},\bg_{t})\right\|\\
	& + \frac{mL^{\bg\bg}}{2}\left\|\eta_{\bg_{t}}\cG_{\bg_{t}}(\bV_{t},\bg_{t})\right\|^{2}.
	\end{aligned}
	\end{equation}
	Following the proof of Theorem \ref{pro:SGD learning rate}, we get the desired conclusion. 
\end{proof}
\section{Upper Bound to $\|\bw_{t}^{(i)}\|^{2}$}\label{app:upper bound}
In this section we respectively give an upper bound to $\|\bw_{t}^{(i)}\|^{2}$ for vanilla GD and SGD. The conclusions are similar to the results in \citep{arora2018theoretical}. 
\begin{proposition}
	For iterates in Theorem \ref{pro:GD learning rate}, we have 
	\begin{equation}
	\small
	\|\bw_{t}^{(i)}\|^{2} \leq \|\bw_{0}^{(i)}\|^{2} + \frac{2(\cL(\bW_{0}, \bg_{0}) - \cL(\bW^{*}, \bg^{*}))}{L_{\bw^{(i)}}}, 
	\end{equation}
	for any $0\leq t \leq T$ and $1\leq i\leq m$.
\end{proposition}
\begin{proof}
	From equation \eqref{eq: gradient bound}, we have 
	\begin{equation}
	\small
	\frac{\|\bw_{t + 1}^{(i)}\|^{2} - \|\bw_{t}^{(i)}\|^{2}}{\eta_{\bw^{(i)}_{t}}^{2}} = \left\|\nabla_{\bw_{t}^{(i)}}\cL(\bW_{t}, \bg_{t})\right\|^{2} \leq 2L_{\bw^{(i)}}(\cL(\bW_{t}, \bg_{t}) - \cL(\bW_{t + 1}, \bg_{t + 1}))
	\end{equation}
	for $1\leq i \leq m$. By the value of $\eta_{\bw^{(i)}}$, we then have
	\begin{equation}
	\small
	L_{\bw^{(i)}}^{2}\left(\|\bw_{t + 1}^{(i)}\|^{2} - \|\bw_{t}^{(i)}\|^{2}\right) \leq  2L_{\bw^{(i)}}(\cL(\bW_{t}, \bg_{t}) - \cL(\bW_{t + 1}, \bg_{t + 1})).
	\end{equation}
	Summing over $t$ from $0$ to $T - 1$, we have
	\begin{equation}
	\small
	\|\bw_{T}^{(i)}\|^{2} - \|\bw_{0}^{(i)}\|^{2} \leq \frac{2(\cL(\bW_{0}, \bg_{0}) - \cL(\bW^{*}, \bg^{*}))}{L_{\bw^{(i)}}}.
	\end{equation} 
	Then we get the conclusion due to the increasing property of $\|\bw_{t}^{(i)}\|^{2}$ with respect to $t$. 
\end{proof}  
\begin{proposition}
	For iterates in Theorem \ref{pro:SGD learning rate}, we have 
	\begin{equation}
	\small
	\mE\left[\|\bw_{t}^{(i)}\|^{2}\right] \leq \|\bw_{0}^{(i)}\|^{2} + O\left(\frac{\sigma^{2}}{\tilde{L}_{\bw^{(i)}}}\right),
	\end{equation}
	for any $0 \leq t \leq T$, and $1\leq i\leq m$.
\end{proposition}
\begin{proof}
	From equation \eqref{eq:bound gradient}, we have 
	\begin{equation}\label{eq:stochastic norm}
	\small
	\begin{aligned}
	\tilde{L}_{\bw^{(i)}}T\mE\left[\|\bw_{t + 1}^{(i)}\|^{2} - \|\bw_{t}^{(i)}\|^{2}\right] &= \mE\left[\left\|\cG_{\bw_{t}^{(i)}}(\bW_{t}, \bg_{t})\right\|^{2}\right]\\
	& = \mE\left[\left\|\cG_{\bw_{t}^{(i)}}(\bW_{t}, \bg_{t}) - \nabla_{\bw_{t}^{(i)}}\cL(\bW_{t}, \bg_{t}) + \nabla_{\bw_{t}^{(i)}}\cL(\bW_{t}, \bg_{t})\right\|^{2}\right]\\
	& \overset{a}{=} \mE\left[\left\|\cG_{\bw_{t}^{(i)}}(\bW_{t}, \bg_{t}) - \nabla_{\bw_{t}^{(i)}}\cL(\bW_{t}, \bg_{t})\right\|^{2} + \left\|\nabla_{\bw_{t}^{(i)}}\cL(\bW_{t}, \bg_{t})\right\|^{2}\right] \\
	& \overset{b}{\leq} \sigma^{2} + \mE\left[\left\|\nabla_{\bw_{t}^{(i)}}\cL(\bW_{t}, \bg_{t})\right\|^{2}\right]
	\end{aligned}
	\end{equation}
	where $a$ and $b$ are due to equation \eqref{eq:bound gradient}. We then have 
	\begin{equation}
	\small
	\begin{aligned}
	\sum\limits_{s=0}^{t}\frac{1}{\tilde{L}_{\bw^{(i)}}\sqrt{T}}\mE\left[\left\|\nabla_{\bw_{s}^{(i)}}\cL(\bW_{s}, \bg_{s})\right\|^{2}\right] \leq \cL(\bW_{0},\bg_{0}) - \cL(\bW^{*},\bg^{*}) + \sum\limits_{s=0}^{t}\frac{\sigma^{2}}{2T}\left(\frac{1}{\tilde{L}_{\bg}} + \sum\limits_{i=1}^{m}\frac{1}{\tilde{L}_{\bw^{(i)}}}\right).
	\end{aligned}
	\end{equation}
	for any $0 \leq t\leq T$ due to equation \eqref{eq:bound gradient}, the monotonic decreasing of $\cL(\bW_{t},\bg_{t})$ and the value of $\eta_{\bw_{t}^{(i)}}$. Let
	\begin{equation}
	\small
	\begin{aligned}
	S_{t}^{(i)} = \sum\limits_{s=1}^{t}\frac{1}{\sqrt{T}}\mE\left[\left\|\nabla_{\bw_{s}^{(i)}}\cL(\bW_{s}, \bg_{s})\right\|^{2}\right] \leq \tilde{L}_{\bw^{(i)}}(\cL(\bW_{0},\bg_{0}) - \cL(\bW^{*},\bg^{*})) + \frac{\sigma^{2}\tilde{L}_{\bw^{(i)}}}{2}\left(\frac{1}{\tilde{L}_{\bg}} + \sum\limits_{i=1}^{m}\frac{1}{\tilde{L}_{\bw^{(i)}}}\right)
	\end{aligned}
	\end{equation}
	Summing over equation \eqref{eq:stochastic norm} with respect to $t$ from $0$ to $T - 1$, and combining the above equation, we conclude that 
	\begin{equation}
	\mE\left[\|\bw_{T}^{(i)}\|^{2} - \|\bw_{0}^{(i)}\|^{2}\right] \leq  \frac{\sigma^{2}}{\tilde{L}_{\bw^{(i)}}} + \frac{S_{T}^{(i)}}{\sqrt{T}} = O\left(\frac{\sigma^{2}}{\tilde{L}_{\bw^{(i)}}}\right).
	\end{equation}
	Thus, we get the conclusion due to the increasing property of $\|\bw_{t}^{(i)}\|^{2}$ with respect to $t$. 
\end{proof}
We respectively bound the $\|\bw_{t}^{(i)}\|^{2}$ for GD and SGD. Due to there is a scalar $\|\bw_{t}^{(i)}\|^{2}$ is in $\|\nabla_{\bw_{t}}\cL(\bW_{t},\bg_{t})\|$, we conclude that both PSI-GD and PSI-SGD accelerate training within a constant scalar.
\section{More Experimental Results}\label{app:experiment}
\subsection{Experiments with Regularizer}
The experimental results in Section \ref{sec5} are conducted without regularizer e.g. $l_{2}$-regularizer. This is due to the regularizer breaks the PSI property of PSI parameters. Even though, it is desired to verify the proposed method within loss with regularizer. 
\par
Thus we conduct the experiments on \texttt{CIFAR10} and \texttt{CIFAR100} to see the performance of the proposed methods on the loss with $l_{2}$-regularizer. The settings follow Section \ref{sec5} in the main part of the paper, expected for the loss has $l_{2}$-regularizer \footnote{The regularizer of SGDG follows the setting in \citep{cho2017riemannian}.}. The results refers to Table \ref{tab:resnet_with_regularizer}. We can observe from the result that adding regularizer significantly improve the performance of SGD and Adam, while the improvement on the manifold based methods is incremental. Even though, optimization on PSI manifold always finds minimum generalize better. We suggest this is because optimizing on the PSI manifold corresponds with a larger learning rate, then the iterates are away from the initialized point. Hence they are able to find flatter minima which is empirically observed to generalize better \citep{keskar2016large}. 
\par
Finally, all the experimental results refers to Figure \ref{fig: cifar10} and \ref{fig: cifar100}.  
\begin{table*}[t!]\small
	\centering
	\caption{Performance of ResNet with regularizer on \texttt{CIFAR10} and \texttt{CIFAR100}. The results are averaged over five independent runs with standard deviation reported.}
	\scalebox{1}{
		{
			\begin{tabular}{|c|ccccc|}
				\hline
				Dataset & \multicolumn{5}{|c|}{CIFAR-10} \\
				\hline
				Algorithm & SGD           & Adam    & SGDG    & PSI-SGD        & PSI-SGDM     \\
				\hline
				ResNet20  & 92.35(±0.10)	&  90.89(±0.12)	 &  91.99(±0.16)	&  92.49(±0.12)	&  \textbf{92.70}(±0.08) \\ 
				ResNet32  & 93.64(±0.21)	&  91.83(±0.34)	 &  92.53(±0.26)	&  93.75(±0.08)	&  \textbf{93.80}(±0.06)	 \\
				ResNet44  & 93.74(±0.10)	&  91.85(±0.23)	 &  93.30(±0.12)	&  \textbf{94.15}(±0.06)	&  93.92(±0.12)  \\
				ResNet56  & \textbf{94.19}(±0.20)	&  91.78(±0.32)	 &  93.37(±0.14)	&  93.98(±0.13)	&   93.99(±0.14)\\ 
				\hline
				Dataset & \multicolumn{5}{|c|}{CIFAR-10} \\
				\hline
				ResNet20  & 68.62(±0.56)	&  65.66(±0.16)	 &  68.67(±0.10)	&  69.62(±0.12)	  &   \textbf{69.64}(±0.06) \\ 
				ResNet32  & 70.67(±0.16)	&  67.72(±0.12)	 &  70.49(±0.23)	&  \textbf{71.21}(±0.16)	  &   71.07(±0.15)	 \\
				ResNet44  & 71.31(±0.30)	&  68.24(±0.25)	 &  71.49(±0.14)	&  \textbf{72.03}(±0.08)	  &   72.00(±0.25)  \\
				ResNet56  & 72.44(±0.26)	&  68.78(±0.32)	 &  71.56(±0.42)	&  \textbf{73.05}(±0.09)	              &  72.62(±0.13)
				\\
				\hline 
	\end{tabular}}}
	\label{tab:resnet_with_regularizer}
\end{table*}

\begin{figure}[t!]\centering
	\begin{minipage}[t]{0.24\linewidth}
		\centering
		\includegraphics[width=1\linewidth]{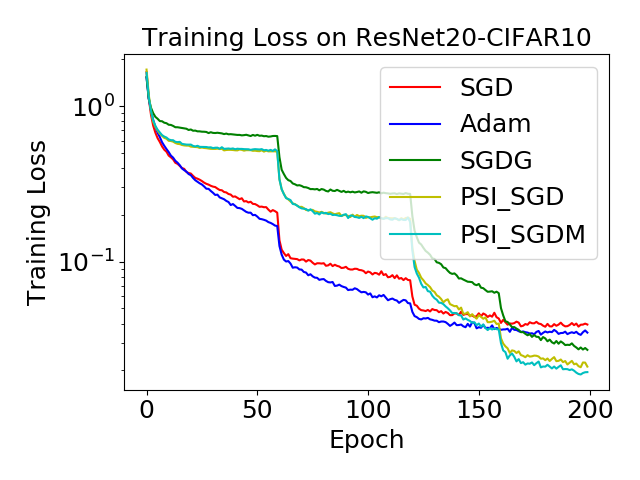}
		\vspace{0.02cm}
		\includegraphics[width=1\linewidth]{./pic/fig_train_without_decay/cifar10_32.png}
		\vspace{0.02cm}
		\includegraphics[width=1\linewidth]{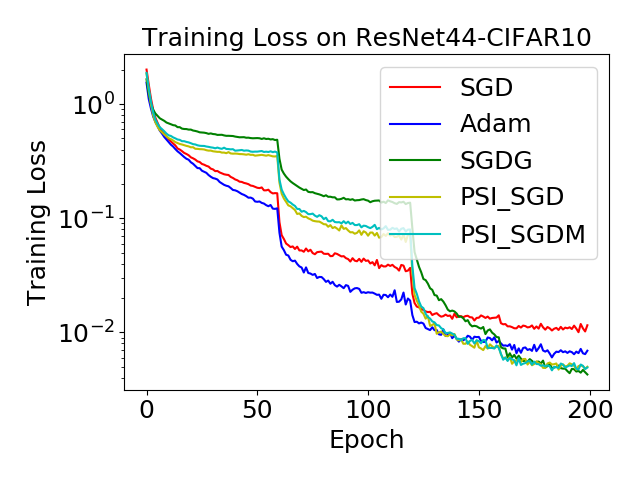}
		\vspace{0.02cm}
		\includegraphics[width=1\linewidth]{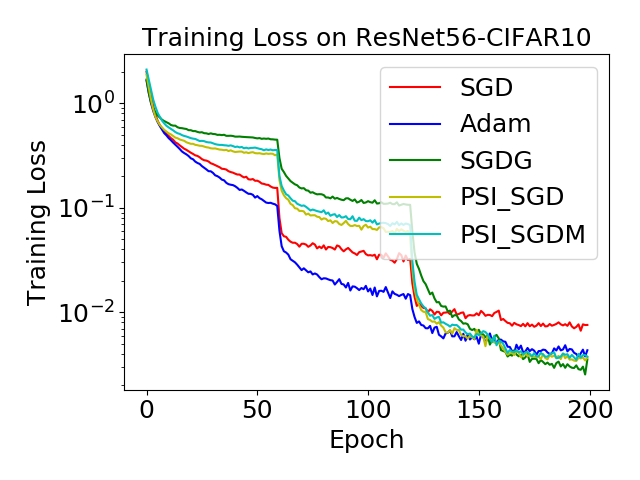}
		\vspace{0.02cm}
	\end{minipage}
	\begin{minipage}[t]{0.24\linewidth}
		\centering
		\includegraphics[width=1\linewidth]{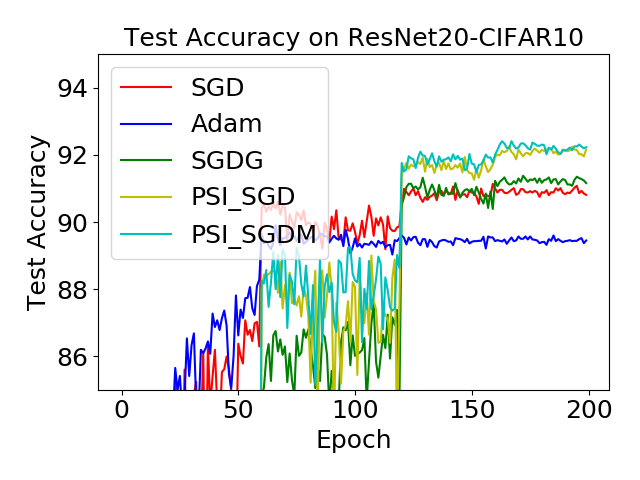}
		\vspace{0.02cm}
		\includegraphics[width=1\linewidth]{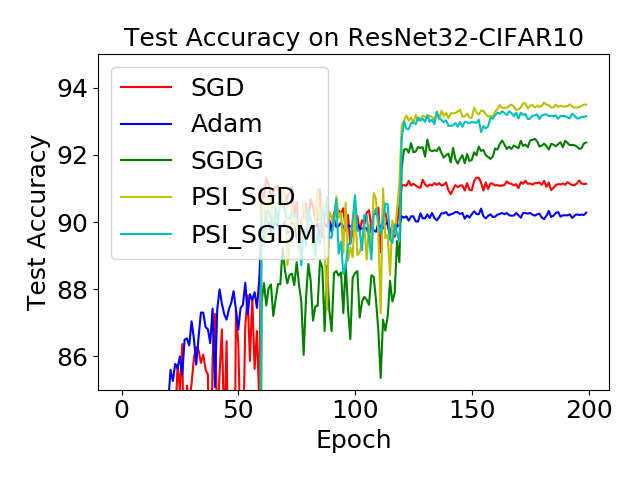}
		\vspace{0.02cm}
		\includegraphics[width=1\linewidth]{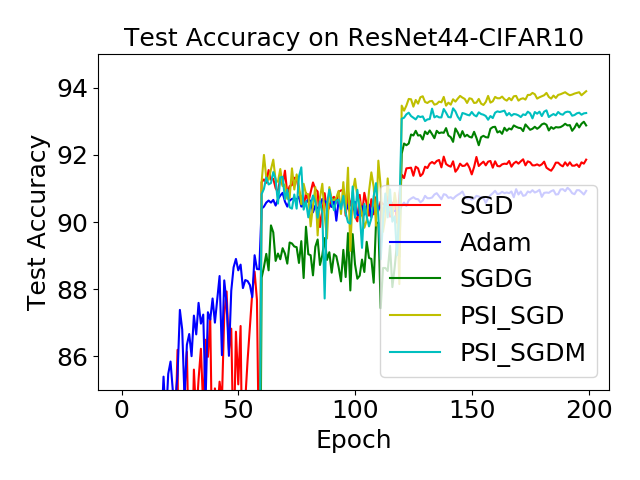}
		\vspace{0.02cm}
		\includegraphics[width=1\linewidth]{./pic/fig_test_without_decay/cifar10_56.png}
		\vspace{0.02cm}
	\end{minipage}
	\begin{minipage}[t]{0.24\linewidth}
		\centering
		\includegraphics[width=1\linewidth]{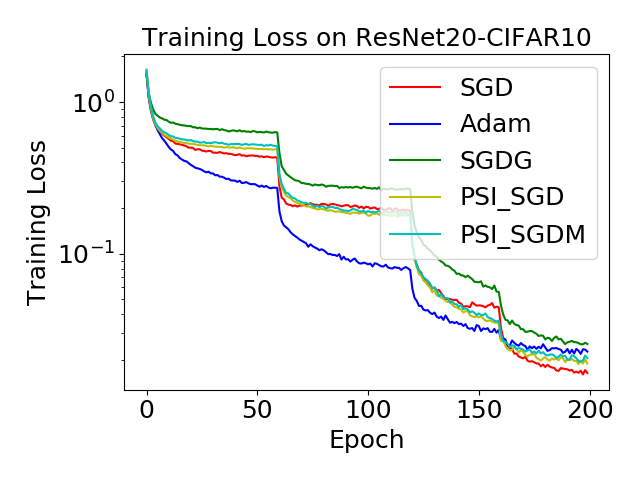}
		\vspace{0.02cm}
		\includegraphics[width=1\linewidth]{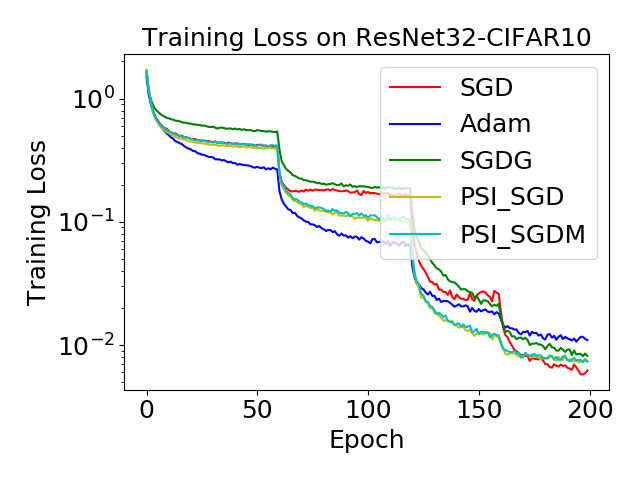}
		\vspace{0.02cm}
		\includegraphics[width=1\linewidth]{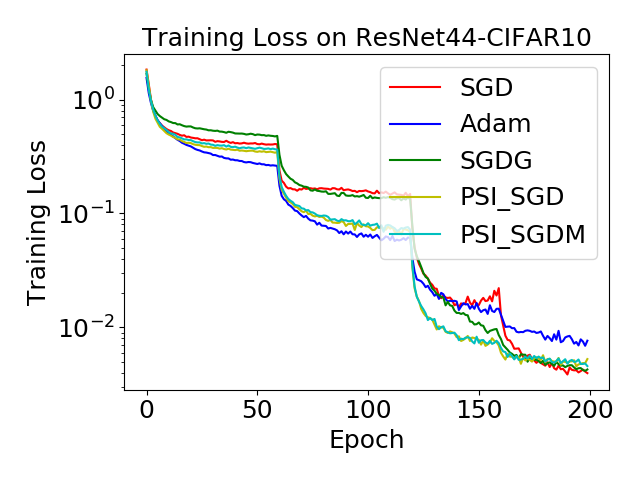}
		\vspace{0.02cm}
		\includegraphics[width=1\linewidth]{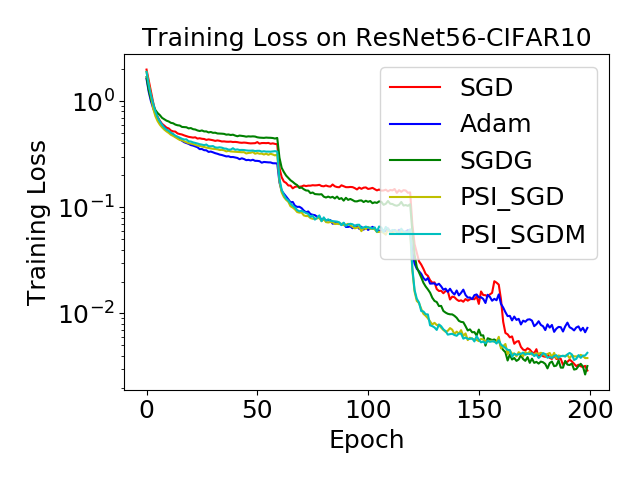}
		\vspace{0.02cm}
	\end{minipage}
	\begin{minipage}[t]{0.24\linewidth}
		\centering
		\includegraphics[width=1\linewidth]{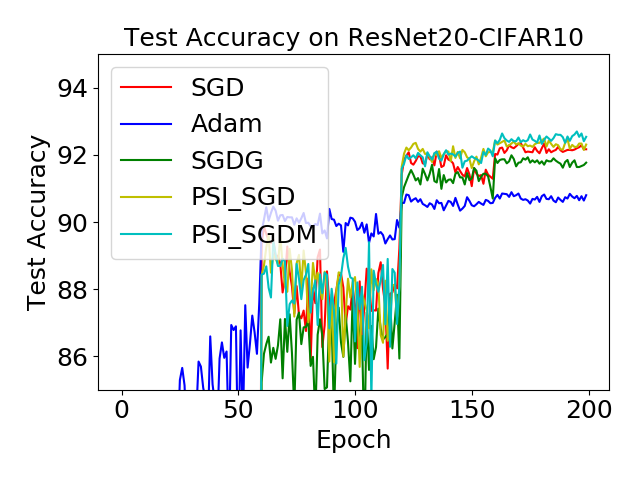}
		\vspace{0.02cm}
		\includegraphics[width=1\linewidth]{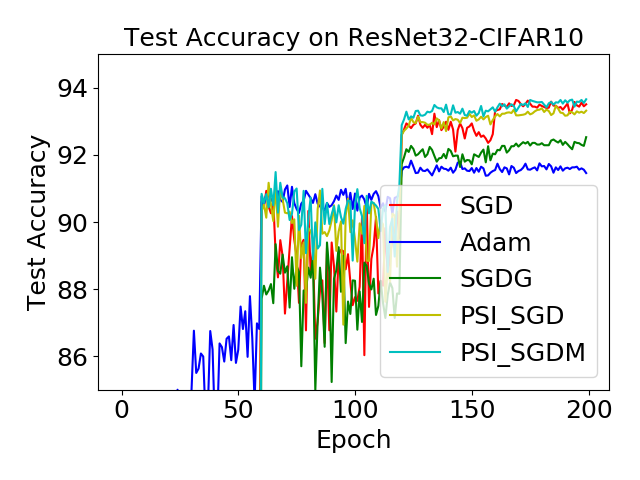}
		\vspace{0.02cm}
		\includegraphics[width=1\linewidth]{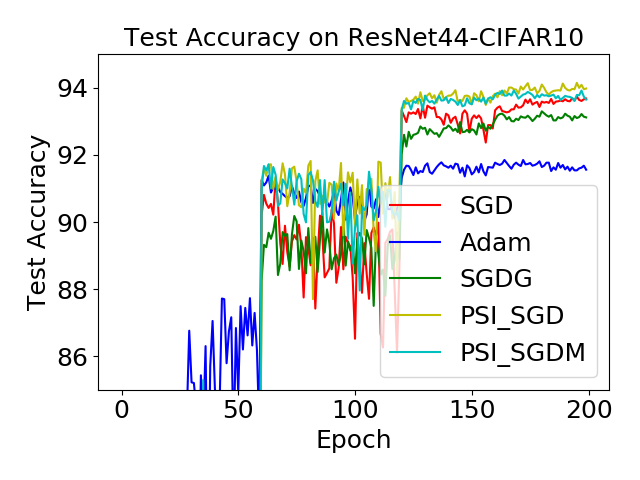}
		\vspace{0.02cm}
		\includegraphics[width=1\linewidth]{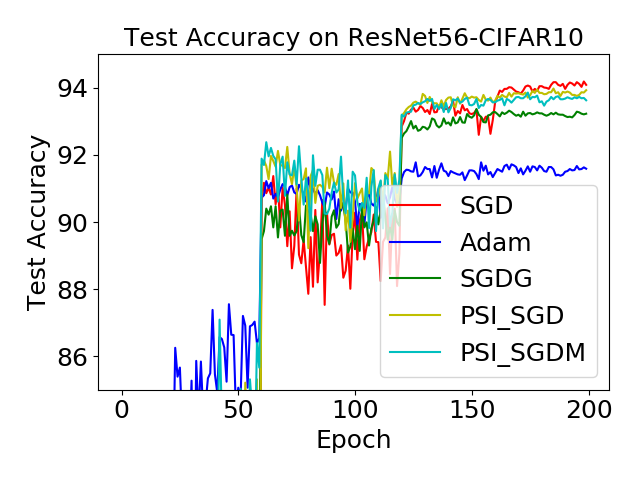}
		\vspace{0.02cm}
	\end{minipage}
	\caption{Results of \texttt{CIFAR10} on various structures of ResNet i.e. 20, 32, 44, 56. The first and the last two columns are respectively obtained with or without regularizer.}
	\label{fig: cifar10}
\end{figure}
\begin{figure}[t!]\centering
	\begin{minipage}[t]{0.24\linewidth}
		\centering
		\includegraphics[width=1\linewidth]{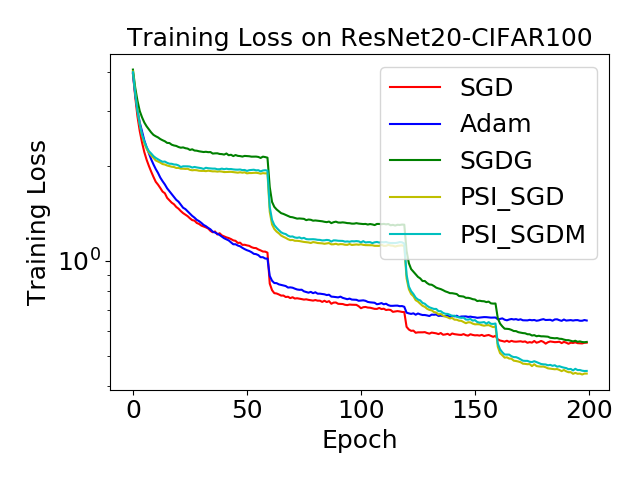}
		\vspace{0.02cm}
		\includegraphics[width=1\linewidth]{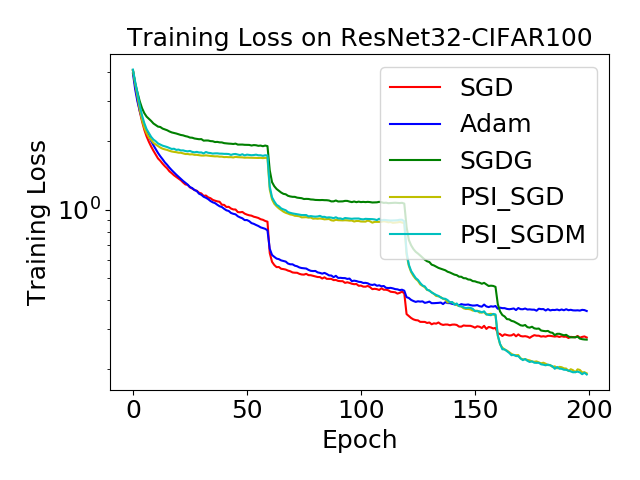}
		\vspace{0.02cm}
		\includegraphics[width=1\linewidth]{./pic/fig_train_without_decay/cifar100_44.png}
		\vspace{0.02cm}
		\includegraphics[width=1\linewidth]{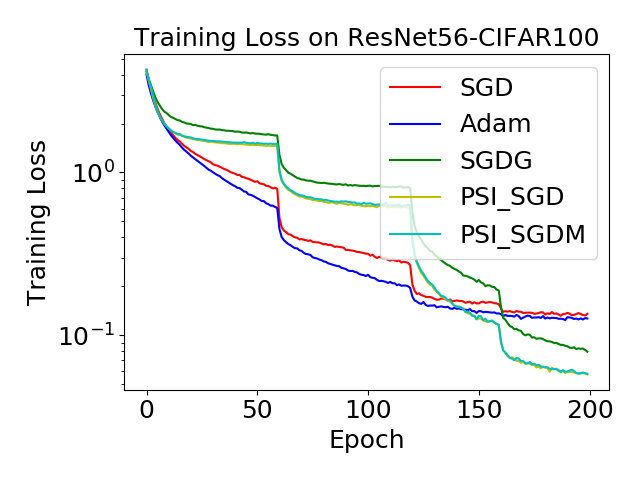}
		\vspace{0.02cm}
	\end{minipage}
	\begin{minipage}[t]{0.24\linewidth}
		\centering
		\includegraphics[width=1\linewidth]{./pic/fig_test_without_decay/cifar100_20.png}
		\vspace{0.02cm}
		\includegraphics[width=1\linewidth]{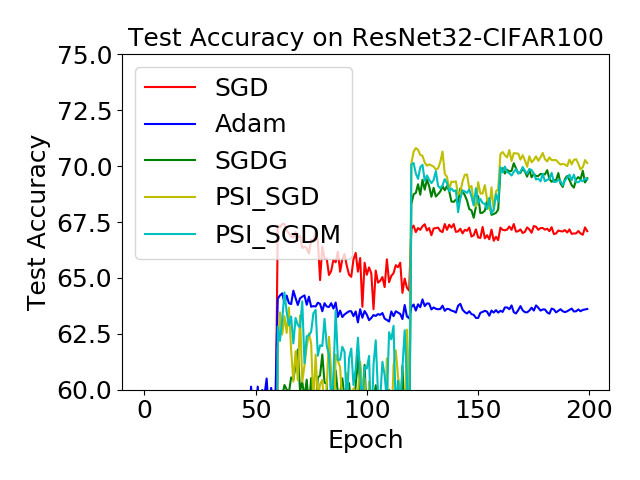}
		\vspace{0.02cm}
		\includegraphics[width=1\linewidth]{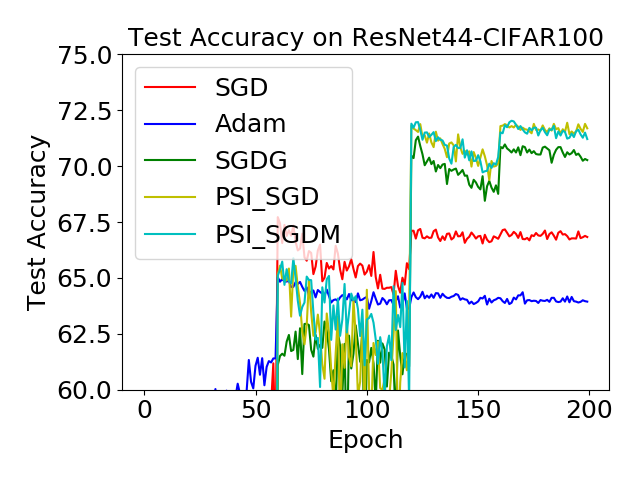}
		\vspace{0.02cm}
		\includegraphics[width=1\linewidth]{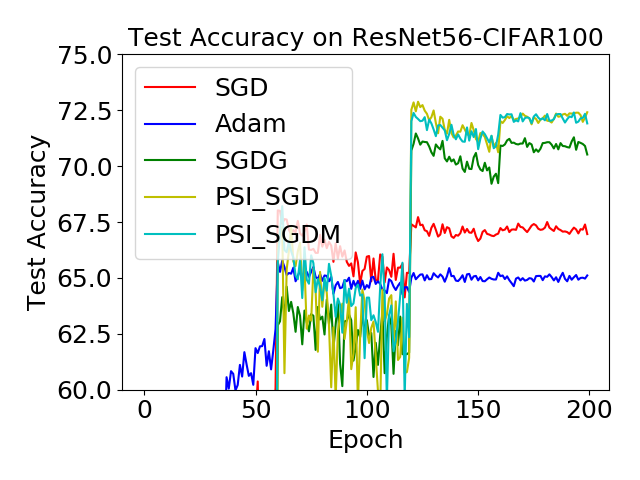}
		\vspace{0.02cm}
	\end{minipage}
	\begin{minipage}[t]{0.24\linewidth}
		\centering
		\includegraphics[width=1\linewidth]{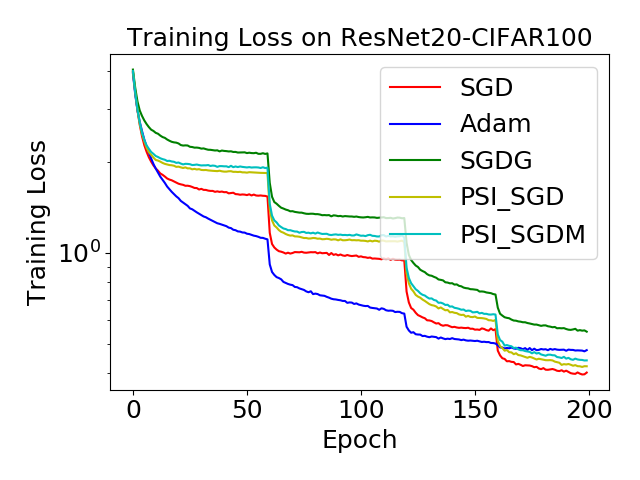}
		\vspace{0.02cm}
		\includegraphics[width=1\linewidth]{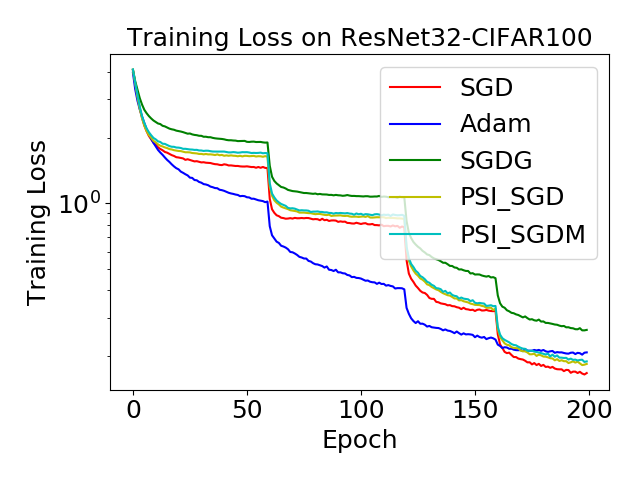}
		\vspace{0.02cm}
		\includegraphics[width=1\linewidth]{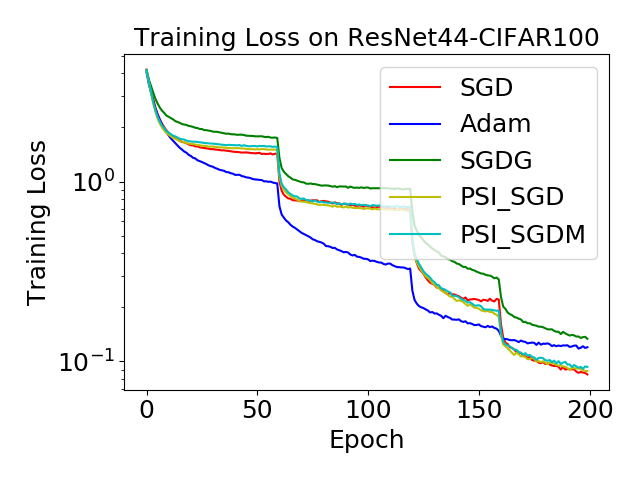}
		\vspace{0.02cm}
		\includegraphics[width=1\linewidth]{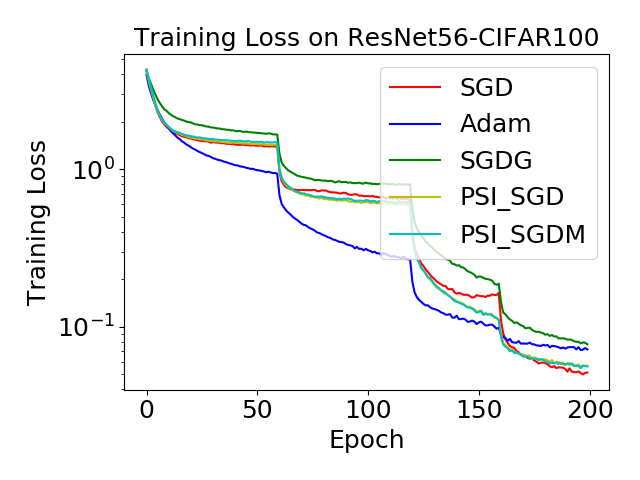}
		\vspace{0.02cm}
	\end{minipage}
	\begin{minipage}[t]{0.24\linewidth}
		\centering
		\includegraphics[width=1\linewidth]{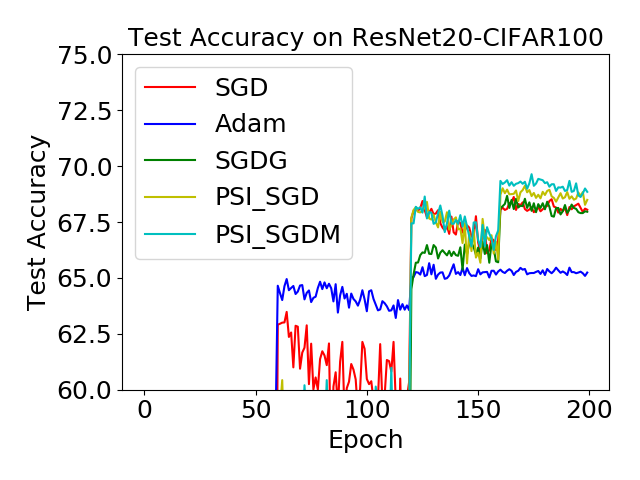}
		\vspace{0.02cm}
		\includegraphics[width=1\linewidth]{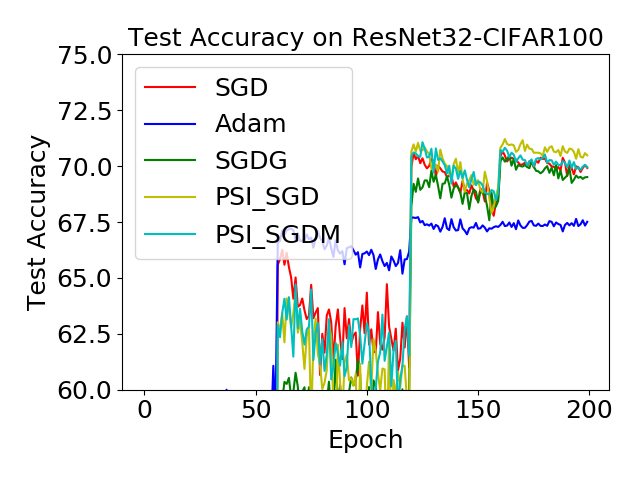}
		\vspace{0.02cm}
		\includegraphics[width=1\linewidth]{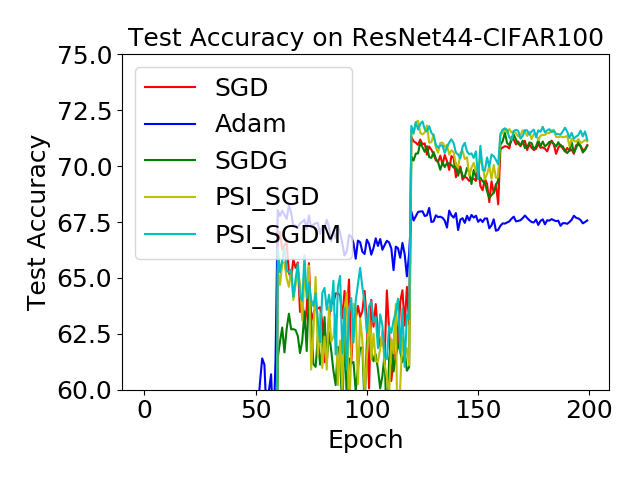}
		\vspace{0.02cm}
		\includegraphics[width=1\linewidth]{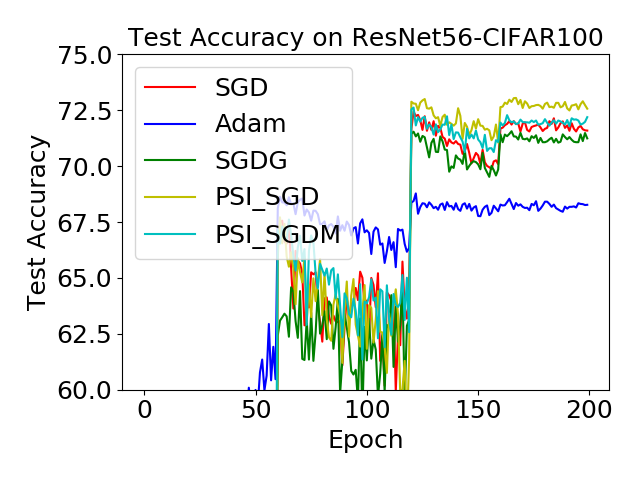}
		\vspace{0.02cm}
	\end{minipage}
	\caption{Results of \texttt{CIFAR100} on various structures of ResNet i.e. 20, 32, 44, 56. The first and the last two columns are respectively obtained with or without regularizer.}
	\label{fig: cifar100}
\end{figure}
\subsection{Hyperparameters}\label{app:hyperparameters}
The PSI parameters are updated by the various methods while non-PSI parameters are all updated by SGD. The hyperparameters refers to Table \ref{tab:hyper}.
\begin{table}[h]
	\caption{The hyperparameters of the experiments}
	\centering
	\begin{tabular}{cccccc}
		\hline
		Hyperparameters &  SGD        & Adam    & SGDG    & PSI-SGD        & PSI-SGDM  \\    
		\hline
		Learning Rate   &  $0.1$      & $0.001$  & $0.2$   & $1.0$          & $0.1$ \\
		Regularizer     &  $0.0005$     & $0.0005$  & $0.1$   & $0.0005$         & $0.0005$ \\
		Batch Size      &  $128$      & $128$   & $128$   & $128$          & $128$  \\
		Momentum        &  $0.9$      & -       & -       & -              & $0.9$  \\
		$\beta_{1}$     &  -          & $0.9$   & -       & -              & -      \\
		$\beta_{2}$     &  -          & $0.999$   & -       & -              & -      \\
		\hline
	\end{tabular}
	\label{tab:hyper}
\end{table}